\documentclass{article}
\usepackage[accepted]{icml2020}

\usepackage{hyperref}
\usepackage{url}
\usepackage{natbib}
\usepackage{booktabs}
\usepackage{color}
\usepackage{placeins}
\usepackage{amsmath}
\usepackage{amsthm}
\usepackage{amssymb}
\usepackage{subcaption}
\usepackage{tikz}
\usepackage{comment}
\usepackage{bm}
\usepackage{array}
\usepackage{times}
\usepackage{scalefnt}
\usepackage{mathtools}
\usepackage{algorithm}
\usepackage{algorithmic}
\usepackage{footnote}
\usepackage{multirow}
\usepackage{thm-restate}

\usepackage[english]{babel}
\usepackage[utf8x]{inputenc}
\usepackage[T1]{fontenc}

\usepackage{amsmath}
\usepackage{amssymb}
\usepackage{graphicx}
\usepackage[colorinlistoftodos]{todonotes}
\newcommand{\Ac}{{\mathcal{A}}}

\newcommand{\Dc}{{\mathcal{D}}}

\newcommand{\Fc}{{\mathcal{F}}}

\newcommand{\Hc}{{\mathcal{H}}}

\newcommand{\Lc}{{\mathcal{L}}}
\newcommand{\Mc}{{\mathcal{M}}}
\newcommand{\Nc}{{\mathcal{N}}}

\newcommand{\Sc}{{\mathcal{S}}}

\newcommand{\Wc}{{\mathcal{W}}}
\newcommand{\Xc}{{\mathcal{X}}}
\newcommand{\Yc}{{\mathcal{Y}}}
\newcommand{\Zc}{{\mathcal{Z}}}

\newcommand{\Eb}{{\mathbb{E}}}

\newcommand{\Gb}{{\mathbf{G}}}

\newcommand{\Yb}{{\mathbf{Y}}}

\newtheorem{prop}{Proposition}

\newtheorem{definition}{Definition}

\newtheorem{lemma}{Lemma}

\newcommand{\Var}{{\mathrm{Var}}}

\newcommand\numberthis{\addtocounter{equation}{1}\tag{\theequation}}

\newcommand{\impliesabove}[1]{\stackrel{\mathclap{\mbox{#1}}}{\implies}}

\usepackage{accents}

\newcommand{\cdf}{{\mathbb{F}}}

\newcommand{\fs}{\mathbf{H}}
\newcommand{\xs}{\mathbf{X}}
\newcommand{\xadvg}{{\xs_\Sc}}
\newcommand{\xadvgi}{{\xs_{\Sc_i}}}
\newcommand{\rs}{{\mathbf{R}}}
\newcommand{\ys}{{\mathbf{Y}}}
\newcommand{\ysx}{{\mathbf{Y}}}
\newcommand{\zs}{{\mathbf{Z}}}
\newcommand{\fd}{h}
\newcommand{\fdr}{\bar{h}}
\newcommand{\fcdf}{{\cdf_{\fs[x](\ysx)}}}
\newcommand{\fcdfs}{{\cdf_{\fs[\xs](\ys)}}}
\newcommand{\fcdfd}{{\cdf_{\fs[x](y)}}}

\newcommand{\cdfjoint}{{\cdf_{\xs\ys}}}
\newcommand{\cdfyx}{{\cdf_{\ys \mid x}}}
\newcommand{\ucdf}{{\cdf_\mathbf{U}}}

\newcommand{\dw}{{d_{\mathrm{W1}}}}
\newcommand{\intr}{{\int_{r=0}^1}}
\newcommand{\intc}{{\int_{c=0}^1}}
\newcommand{\err}{{\mathrm{err}}}

\renewcommand{\Dc}{{\mathbf{\mathcal{D}}}}
\newcommand{\eqdef}{\mathrel{\stackrel{\makebox[0pt]{\mbox{\scriptsize def}}}{=}}}
\newcommand{\simiid}{\mathrel{\stackrel{\makebox[0pt]{\mbox{\scriptsize i.i.d.}}}{\sim}}}

\def\shownotes{1}  \ifnum\shownotes=1

\newcommand{\authnote}[2]{$\ll$\textsf{\footnotesize #1: #2}$\gg$}
\else
\newcommand{\authnote}[2]{}
\fi

\icmltitlerunning{Individual Calibration with Randomized Forecasting}

\begin{document}

\twocolumn[
\icmltitle{Individual Calibration with Randomized Forecasting}

% It is OKAY to include author information, even for blind
% submissions: the style file will automatically remove it for you
% unless you've provided the [accepted] option to the icml2019
% package.

% List of affiliations: The first argument should be a (short)
% identifier you will use later to specify author affiliations
% Academic affiliations should list Department, University, City, Region, Country
% Industry affiliations should list Company, City, Region, Country

% You can specify symbols, otherwise they are numbered in order.
% Ideally, you should not use this facility. Affiliations will be numbered
% in order of appearance and this is the preferred way.
\icmlsetsymbol{equal}{*}

\begin{icmlauthorlist}
\icmlauthor{Shengjia Zhao}{to}
\icmlauthor{Tengyu Ma}{to}
\icmlauthor{Stefano Ermon}{to}
\end{icmlauthorlist}

\icmlaffiliation{to}{Computer Science Department, Stanford University}

\icmlcorrespondingauthor{Shengjia Zhao}{sjzhao@stanford.edu}
\icmlcorrespondingauthor{Tengyu Ma}{tengyuma@stanford.edu}
\icmlcorrespondingauthor{Stefano Ermon}{ermon@stanford.edu}

% You may provide any keywords that you
% find helpful for describing your paper; these are used to populate
% the "keywords" metadata in the PDF but will not be shown in the document
\icmlkeywords{Machine Learning, ICML}

\vskip 0.3in
]

% this must go after the closing bracket ] following \twocolumn[ ...

% This command actually creates the footnote in the first column
% listing the affiliations and the copyright notice.
% The command takes one argument, which is text to display at the start of the footnote.
% The \icmlEqualContribution command is standard text for equal contribution.
% Remove it (just {}) if you do not need this facility.

%\printAffiliationsAndNotice{}  % leave blank if no need to mention equal contribution
\printAffiliationsAndNotice{} % otherwise use the standard text.

\begin{abstract}
Machine learning applications often require calibrated predictions, e.g. a 90\% credible interval should contain the true outcome 90\% of the times. However, typical definitions of calibration only require this to hold on average, and offer no guarantees on predictions made on individual samples. Thus, predictions can be systematically over or under confident on certain subgroups, leading to issues of fairness and potential vulnerabilities. 
We show that calibration for individual samples is possible in the regression setup if the predictions are randomized, i.e. outputting randomized credible intervals. Randomization removes systematic bias by trading off bias with variance. We design a training objective to enforce individual calibration and use it to train randomized regression functions. The resulting models are more calibrated for arbitrarily chosen subgroups of the data, and can achieve higher utility in decision making against adversaries that exploit miscalibrated predictions.  
\end{abstract}

\section{Introduction}

% Many machine learning applications not only require high accuracy, but also accurate uncertainty estimation. Ideally we want to predict the true probability 

\begin{figure*}
    \centering
    \begin{tabular}{cc}
    \includegraphics[width=0.48\linewidth]{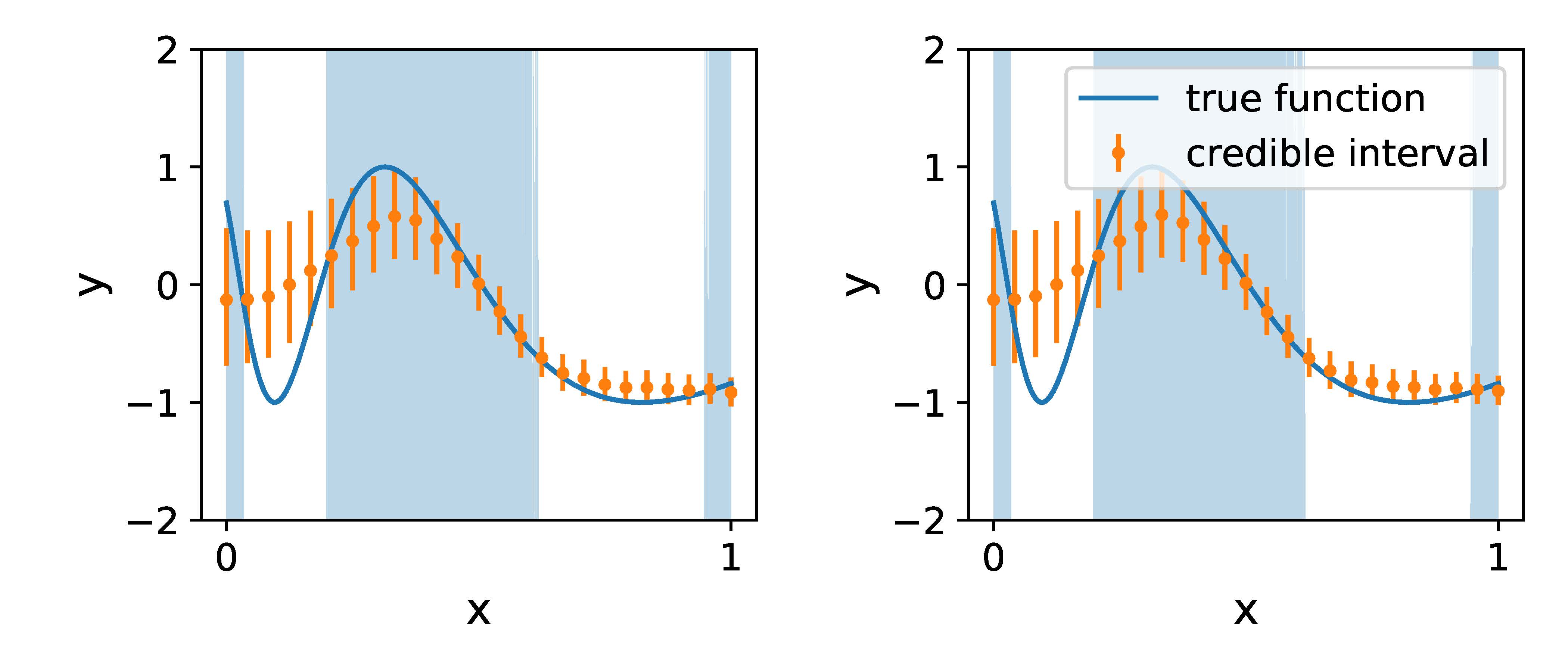} &
    \includegraphics[width=0.48\linewidth]{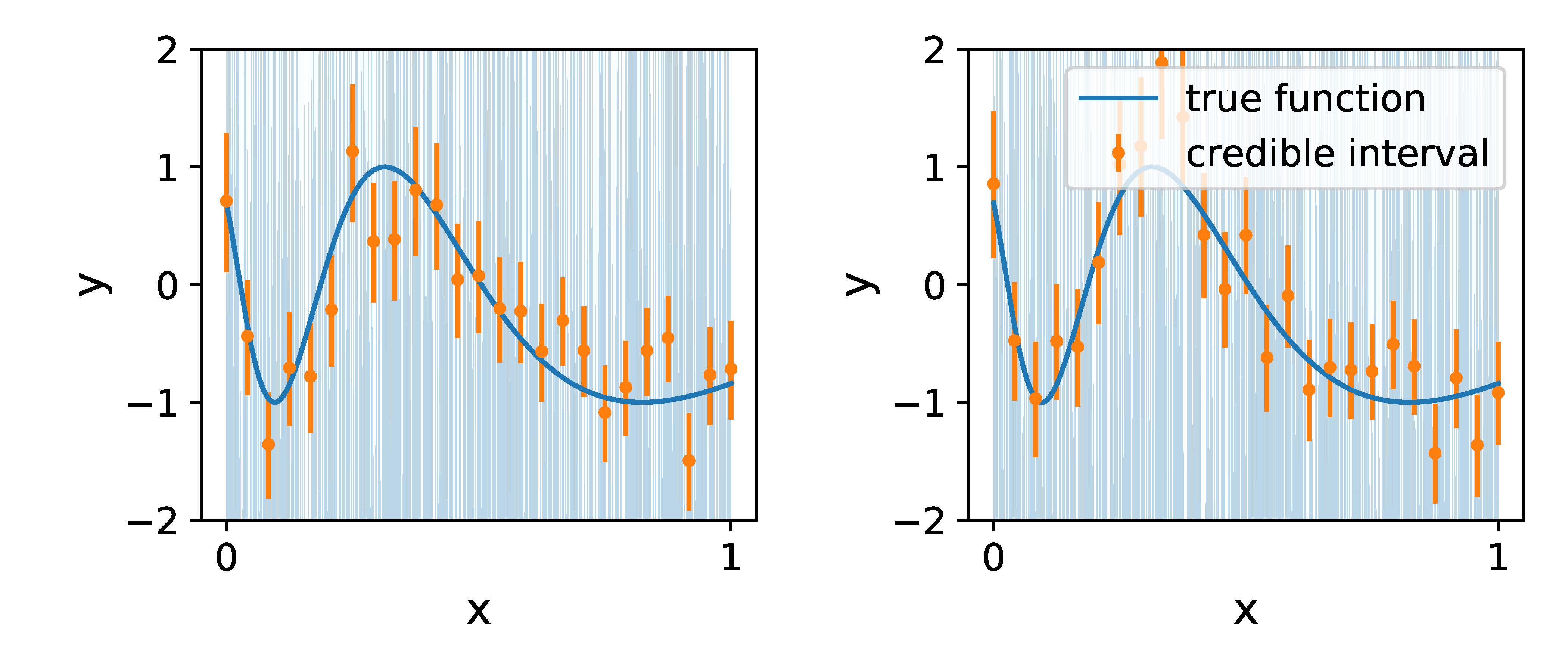} 46
    % \parbox{0.5\linewidth}{\centering Predicted credible interval for standard calibrated forecaster (no randomization)} &
    % \parbox{0.5\linewidth}{\centering Predicted credible interval for individually calibrated forecaster (with randomization)} 
    \end{tabular}
    \vspace{-5mm}
    \caption{Example of deterministic forecaster that is calibrated on average (\textbf{left}) and randomized forecaster that is individually calibrated (\textbf{right}). We plot the {\color{orange} 80\% credible interval} and the {\color{orange} median (orange dot)} the forecaster outputs, and {\color{cyan} shade in cyan the area} where the {\color{orange} predicted median} is less than the {\color{blue} true function value}. \textbf{Left}: the deterministic forecaster outputs a fixed credible interval (the left 2 plots are identical) and and can be miscalibrated on sub-groups of the data samples (e.g. it is not calibrated for the samples in the {\color{cyan} shaded area}, because the {\color{orange} predicted median} is always less than the {\color{blue}true function value}). %\s{why is deterministic plotted twice?}\sj{To highlight it's deterministic..... and to fill up the extra space}
    \textbf{Right}: The randomized forecaster outputs a different credible interval each time (the right 2 plots are different), and can remove systematic miscalibration on sub-groups of data samples. (The shaded area becomes random.)%\tnote{edited slightly}
    } %(e.g. the shaded area becomes random).
    %\s{figs an captions should be self-contained. somebody should be able to just look at the figures and captions and understand the main ideas/story of the paper}
    % \tnote{better caption.}}
    \label{fig:toy_example}
\end{figure*}

Many applications of machine learning, such as safety-critical systems and medical diagnosis, require accurate estimation of the uncertainty associated with each prediction. %\s{why? mention safety} 
Uncertainty is typically represented using a probability distribution on the possible outcomes.
To reflect the underlying uncertainty, these probabilities should be calibrated~\citep{cesa2006prediction,vovk2005algorithmic,guo2017calibration}. In the regression setup, for example, the true outcome should be below the predicted 50\%  quantile (median) roughly 50\% of the times~\citep{kuleshov2018accurate}. 

%In the regression setup, a typical definition of calibration is: for $r$ percent of the samples, the true value should be below the predicted $r$-th quantile. 

However, even when the probability forecaster is calibrated, it is possible that the true outcome is \emph{always} below the predicted median on a subgroup (e.g., men), and \emph{always} above the predicted median for another (e.g., women). 
In other words, the standard notion of calibration is a property that needs to hold only \emph{on average across all predictions} made (i.e., on the set of all input data points). 
The predicted probabilities can still be highly inaccurate for \emph{individual} data samples. 
These predictions can lead to unfair or otherwise suboptimal decisions. For example, a bank might over-predict credit risk for one gender group and unfairly deny loans, or under-predict credit risk for a group that can then exploit this mistake to their advantage. 
% In addition, they could lead to high 
% For example, when under-predicting credit risk for loan applicants lead to bad lending decisions. To bound the risk of a bad decision, the forecaster must be calibrated on the group of people who will actually apply --- instead of average calibration on the general population.  %to bound the risk of a bad decision, and average calibration on the general population is not useful. 
%
%in autonomous driving, under-predicting risk of hitting a truck leads to dangerous decisions, but not so much for hitting a curb. The forecaster must be calibrated on the subset of trucks . %good decision loss. %needs  Average calibration cannot guarantee the probability of under-prediction for trucks --- that would require calibration on the subset of trucks. \kh{I think you want to say that average calibration may lead to under-predicting risk for hitting a truck} %\s{again, bad examples. come up with something more credible. if possible, the same setup as the experiments}
Group calibration~\citep{kleinberg2016inherent} partly addresses the short-comings of average calibration by requiring calibration on pre-specified groups 
%\s{give example of what a groupaverage is} 
(e.g. men and women). In particular, \citep{hebert2017calibration} achieves calibration on any group that can be computed from input by a small circuit. However, these methods are not applicable when the groups are unknown or difficult to compute from the input. For example, groups can be defined by features that are unobserved e.g. due to personal privacy. % that might features such as orientation or political view can be sensitive and unobserved
%which is sometimes difficult or even impossible. 
%
%For example, groups that require fair treatment are often sensitive and private. 
%
% people who exploit a bank's mistake are also difficult to predict. 
%it is difficult for a bank to know which group of people could exploit its mistakes in calibration. %~\citep{kleinberg2016inherent,hebert2017calibration}.  
%which is only applicable if every group we want to calibrate is labeled \s{??} 
%(or in the case of \cite{} computable by a simple function from existing features).
%\kh{what do you mean by "circuit"?}. 
%They are not applicable if the groups are unknown \kh{redundant from what you said above} or depend on unobserved features \kh{what does this mean?, there's no example for unobserved features}. 
%For example, this could be impossible  in an environment with adversarial agents and the groups needing calibration depend on adversary strategy. 
%For example, if a bank predicts credit risk to decide on loan approval, any group of potential customers with under-estimated credit risk can apply for the loan and exploit the bank --- the bank is not calibrated for this adversarial group. 
%and approves more loans to unqualified customers. 
%To avoid getting exploited, 

Ideally, we would like forecasters that are calibrated on each individual sample. % --- this group is determined by the prediction function and unknown in advance. \kh{I don't get this adversarial environment example. } %\s{explain better, give example}
%We explore the possibility of calibration for each individual sample. 
Our key insight is that individual calibration is possible when the probability forecaster is itself randomized.
%\kh{I think you'd wanna be consistent with language. This is the first time you use the word "forecaster" although it appears throughout the paper and the title} 
Intuitively, 
%fixed forecasters predicts a fixed median, and the true value must be above or below the fixed median with 50\% probability. 
a randomized forecaster can output random probabilistic forecasts (e.g., quantiles) --- it is the predicted quantile that is randomly above or below a fixed true value with the advertised probability (see Figure~\ref{fig:toy_example}). %\s{this is tricky and won't be clear to reader}
%the true value approximately half of the times, and below the true value half of the times. 
%If the median is itself randomized, the the true value can be above/below it with 1/2 probability. 
 Randomization can remove systematic miscalibration on any group of data samples. %This is . %\s{referring to test data also seems imprecise, since you want generalization. this isn't some kind of transductive learning setting}
%
%It's possible to trivially construct a useless forecaster that achieve individual calibration.
%can be trivially achieved by a practically useless forecaster. 
Useful forecasters also need to be sharp --- the predicted probability should be concentrated around the true value. We design a concrete learning objective that enforces individual calibration. Combined with an objective that enforces sharpness, such as traditional log-likelihood, we can learn forecasters that trade-off calibration and sharpness \textit{Pareto optimally}. The objective can be used to train any prediction model that takes an additional random source as input, such as deep neural networks. %Compared to trivial randomization (add noise to ``fake'' individual calibration) 
%\s{need to pitch this differently. there aren't that many randomized models}) can learn on this objective \s{?} and achieve better individual calibration. 

We assess the benefit of forecasters trained with the new objective on two applications: fairness and decision making under uncertainty against adversaries. 
Calibration on protected groups traditionally has been a definition for fairness~\citep{kleinberg2016inherent}. On a UCI crime prediction task, we show that forecasters trained for individual calibration achieve lower calibration error on protected groups without knowing these groups in advance. %We support this claim on a UCI crime prediction dataset.    %For fairness we consider 
%\s{datasets dont predict} 
%sensitive feature (e.g. crime rate) prediction. %A typical condition for fairness is that the prediction is calibrated for every protected group (e.g. group of men and group of women). 
%\s{give example} 
% Forecasters trained with the individual calibration objective achieve lower calibration error on all groups of data samples, including several important groups (such as gender groups). %\kh{what does "worst group" mean? What are "interpretable" ones?} 
% Compared to group calibration, these groups do not have to be known in advance.

For decision making we consider the Bayesian decision strategy --- choosing actions that minimize expected loss under the predicted probabilities. We prove strong upper bounds on the decision loss when the probability forecaster is  calibrated on average or  individually. However, when the input data distribution changes, forecasters calibrated on average lose their guarantee, while  %However, average calibration failure to 
% %\kh{does calibrated on the input data distribution here mean individual calibration or average calibration?} 
% However, with distribution shift 
% when decision making strategies are deployed in the real world, the input data distribution could be different from training data, or even aversarially generated. We show that 
individually calibrated forecasters are less affected. % by distribution shift. %\kh{do you mean you show this?} %strong guarantees on the performance of a Bayesian decision. upper bound that probability that the true loss exceed the Bayesian loss.
%a strong guarantee on its per has formalize a general strategy for decision making with calibrated predictions. 
%If the predicted probability is only average calibrated, the decision loss is guaranteed only when the data 
%However, the decision loss can be large when the true environment contains adversaries that exploit miscalibration on groups of data points (or equivalently we would like it to perform well regardless of unknown aspects of the environment). 
We support these results by simulating a game between a bank and its customers. The bank approves loans based on predicted credit risk, 
%\kh{transition to the example is awkward} 
and customers exploit any mistake of the bank, i.e. the distribution of customers change adversarially for the bank. % (so customer distribution differs from training data). %Because the customers play rationally, the  % i.e. customers more likely apply for the loan when their true credit risk is high but get approved with high probability. % but a high chance of getting it. 
%customers with under-predicted credit risk are more likely to apply. 
We observe that the bank incurs lower loss when the credit risk forecaster is trained for individually calibration. % forecaster achieve lower loss compared to decisions based on a forecaster that is only average calibrated on the training data. %\kh{what is utility?}

%\kh{Here's what I got from this intro: calibration is important -> average calibration is bad for decision making -> group calibration is one possible solution, but has limits -> we propose to use randomized forecasters for individual calibration -> we evaluate our solution on fairness and bayesian decision making. If this is what you intended, then I think the overall structure of the intro is clear.}

%\kh{While I see what each paragraph sought to do, the job is not done well in some paragraphs. In general I think what's crucial is to focus on a few really compelling examples instead of listing a lot of under-developed examples.}

%\kh{One other general comment is to use consistent language across the intro. Forecaster was an unfamiliar term, and might be worth defining it informally in the intro.}

% We run simulations on two concrete example: approving credit card / loan for banks, and ensuring safety in autonomous driving. 
%\s{this paragraph sounds sketchy, will need careful wording to be believable}
%The interpretable groupaverages include groups of particular race/ while average calibrated forecasters achieve poor calibration on groups that can be interpreted as e.g. immigrants or race. 

\begin{figure*}
    \centering
    \hspace{-4mm}
    \begin{tabular}{c|cp{.1cm}p{.1cm}cp{.1cm}p{.1cm}cp{.1cm}p{.1cm}c}
    % & \parbox{0.13\linewidth}{\centering Perfect forecaster \\ (Definition~\ref{def:perfect})} 
    % & $\implies$ 
    & \parbox{0.14\linewidth}{ \centering Individual \\ calibration \\ (Definition~\ref{def:individual_calibration})} 
    & $\impliesabove{\text{Thm~\ref{thm:pai_vs_calibration}}}$ &
    & \parbox{0.14\linewidth}{\centering Adv group \\ calibration \\(Definition~\ref{def:adv_group_calibration})} 
    & {$\implies$} &
    & \parbox{0.12\linewidth}{\centering Group \\ calibration \\ (Definition~\ref{def:group_calibration})} 
    & $\implies$ &
    & \parbox{0.12\linewidth}{\centering Average \\ calibration\\ (Definition~\ref{def:average_calibration})} 
    \\ \hline
    \parbox{0.125\linewidth}{\centering R.V. that needs \\ to be uniform} 
    & $\fs[x](\ysx)$ $\forall x \in \Xc$ 
    &&& \parbox{0.12\linewidth}{\centering $\fs[\xadvg](\ys)$ $\forall S \subset \Xc$} 
    &&& \parbox{0.12\linewidth}{\centering $\fs[\xadvgi](\ys)$ \\  $i=1, \cdots, n $} 
    &&& $\fs[\xs](\ys)$ %\tnote{X instead of $\tilde{X}$}
    \\ \hline
    Deterministic
    & \multicolumn{2}{c}{Not Achievable (Prop~\ref{prop:impossible_deterministic})}
    & \multicolumn{3}{c}{Not achievable (Prop~\ref{prop:impossible_deterministic_advgroup})} 
    & \multicolumn{3}{c}{Achievable}
    && Achievable
    \\ \hline
    Randomized 
    & \multicolumn{2}{c}{Achievable* (Thm~\ref{thm:paicmono_imply_paic})}
    & \multicolumn{3}{c}{Achievable* (Thm~\ref{thm:pai_vs_calibration})}
    && Achievable  
    &&&  Achievable  
    \end{tabular}
    \caption{Relationships between different notions of calibration, ordered from strongest (individual calibration) to the weakest (average calibration). *With caveats in certain situations, see additional discussion in Section~\ref{sec:appendix_additional_discussion}. 
    %If a forecaster satisfies calibration in the stronger sense, it automatically satisfies calibration in the weaker sense, but the converse is false in general.
    } %\tnote{I wonder if we could remove the perfect forecast in this talbe}
    \label{fig:comparison}
\end{figure*}

\section{Preliminary: Forecaster and Calibration}
%\tnote{check "Probability Forecast" or "Probability Forecasting"}

%\tnote{how about section title = "Preliminary: Probability Forecast and Calibration"}

%\tnote{in general, it's perhaps better to give readers lower expectation }
\subsection{Notation}
% Let $\Sc$ be an interval in $\mathbb{R}$, we use $\Fc(\Sc)$ to denote the set of cumulative distribution functions (CDF) on $\Sc$. In other words, $\Fc(\Sc)$ is the set of functions $\Sc \to [0, 1]$ that satisfy the conditions of a CDF.  \s{do you need to worry about measures on S?}. We say the CDF is continuous if the function $\Sc \to [0, 1]$ is continuous.
We use bold capital letters $\xs, \ys, \zs, \fs$ to denote random variables, lower case letters $x, y, z, h$ to denote fixed values, and $\Xc, \Yc, \Zc, \Hc$ to denote the set of all possible values they can take.

For a random variable $\zs$ on $\Zc$ we will use $\cdf_\zs$ to denote the distribution of $\zs$, and denote this relationship as $\zs \sim \cdf_\zs$. If $\Zc$ is an interval in $\mathbb{R}$ we 
%additionally 
overload $\cdf_\zs$ to denote the cumulative distribution function (CDF) $\Zc \to [0, 1]$ of $\zs$. 
%In other words, $\cdf_\zs$ is a function $\Zc \to [0, 1]$ that satisfy the conditions of a CDF. 
%We also use the notation $z \sim \cdf_\zs$ to denote a sample drawn from the distribution $\cdf_\zs$. % and $\xs \sim \cdf_\xs$ denote random variable $\xs$ has distribution $\cdf_\xs.$ %When we use denote the probability of an event (e.g. $\Pr[X \geq x])$, the probability is with respect to the distribution of the random variables (i.e. the variables in capital letters). 

$\xs$ is a random variable on $\Xc$, if $\Sc \subset \Xc$ is a measurable set and $\Pr[\xs \in \Sc] > 0$, we will use the notation $X_\Sc$ as the random variable distributed as $\xs$ conditioned on $\xs \in \Sc$. %\tnote{I am hesitant about this definition. Will read more and get back to this.} \tnote{$S$ has no-zero measure. }
%\s{I suspect you need some measure theory (sigma algebra) machinery for this to be well defined, e.g need to make sure S is measurable. a conditional random var is also a tricky object}
%\sj{Is it sufficient to say that $\Sc$ has non-zero measure?}

Let $\Yc$ be an interval in $\mathbb{R}$, we use $\Fc(\Yc)$ denote the set of all CDFs on $\Yc$. 
We use $d: \Fc([0, 1]) \times \Fc([0, 1]) \to \mathbb{R}$ to denote a distance function between two CDFs on $[0, 1]$. For example, the Wasserstein-1 distance is defined for any $\cdf, \cdf' \in \Fc([0, 1])$ as 
\[ d_{W1}(\cdf, \cdf') = \int_{r=0}^1 \lvert \cdf(r) - \cdf'(r) \rvert dr \]
This is the distance we will use throughout the paper. We provide results for other distances in the appendix. 

%\tnote{define Wasserstein distance}
\subsection{Problem Setup}
Given an input feature vector $x \in \Xc$ we would like to predict a distribution on the label $y \in \Yc$. We only consider regression problems where $\Yc$ is an interval in $\mathbb{R}$. 

Suppose there is a true distribution $\cdf_\xs$ on $\Xc$, and a random variable $\xs \sim \cdf_\xs$. 
%use $\xs$ to denote the random variable with that distribution and $\cdf_\xs$ its CDF. 
For each $x \in \Xc$, we also assume there is some true distribution $\cdfyx$ on $\Yc$, and a random variable $\ys \sim \cdfyx$. %Recall that we overload notation and use $\cdfyx$ to denote a CDF. %\kh{didn't you already define this overloading?} 
As a convention, the random variable $\ys$ only appears in any expression along side $x$ or $\xs$. Its distribution is always defined conditioned on $x$ (or after we have randomly sampled $x \sim \cdf_\xs$). %In the former case, $\ys$ has distribution $\ys \sim \cdfyx$; in the latter case, we first sample $x \sim \cdf_\xs$ then $\ys \sim \cdfyx$. 

% $\times \Yc$ with CDF $\cdf^*$. We use $\xs, \ys$ to denote the random variables jointly distributed according to that CDF (i.e. $\cdf_{\xs\ys} = \cdf^*$). For fixed values of $x \in \Xc$, we use $\ysx$ to denote $\ys$ conditioned on $\xs = x$, and denote the CDF of the conditional distribution as $\cdf_{\ysx}$. We also denote the CDF of the marginal distribution of $\xs$ as $\cdf_\xs$.

A \textbf{probability forecaster} is a function $\fd: \Xc \to \Fc(\Yc)$ that maps an input $x\in \Xc$ to a continuous CDF $h[x]$ over $\Yc$.  Note that $\fd[x]$ is a CDF, i.e. it is a function that takes in $y\in \Yc$ and returns a real number $\fd[x](y) \in [0, 1]$. We use $[\cdot]$ to denote function evaluation for $x$ and $(\cdot)$ for $y$. %\kh{this double bracket notation is a bit confusing, but after going through the theory I think it's ok.}

Let $\Hc \eqdef \lbrace \fd: \Xc \to \Fc(\Yc) \rbrace$ be the set of possible probability forecasters.
%\s{explain bracket notation}
We consider \textbf{randomized} forecasters $\fs$ which is a random function taking values in $\Hc$. %We will use $\fs$ to denote the random  %\s{bad notation.., should be something like $X \times \Omega \rightarrow$ .at least once, define it properly, and then say you use some shorthand for compactness} 
% and $H[x](y)$ is a random variable that takes values in $[0, 1]$. 

To clarify notation, $\fs[\xs](\ys), \fs[x](\ysx)$ and $\fs[x](y)$ are all random variables taking values in $[0, 1]$, but they are random variables on different sample spaces. 
\begin{itemize}
    \item $\fs[\xs](\ys)$ is a random variable on the sample space $\Hc \times \Xc \times \Yc$ --- All of $\fs, \xs, \ys$ are random.
    \item $\fs[x](\ysx)$ is a random variable on the sample space $\Hc \times \Yc$, while $x$ is just a fixed value in $\Xc$.  
    \item $\fs[x](y)$ is a random variable on the sample space $\Hc$, while $x, y$ are just fixed values in $\Xc \times \Yc$.
\end{itemize}
Similarly $\fd[\xs](\ys), \fd[x](\ys)$ are also random variables taking values in $[0, 1]$, while $h[x](y)$ is just a number in $[0, 1]$ (there is no randomness). %In general $\fs/\fd, \xs/x, \ys/y$ all can be either random (represented by upper case letter) or fixed (represented by lower case letter).
This difference will be crucial to distinguishing different notions of calibration. 
%We suggest the readers fully digest the difference before moving on. 

We use $\fcdfs, \fcdf$ and $\fcdfd$, etc, to denote the CDF of these $[0, 1]$-valued random variables. These are in general different CDFs because the random variables have different distributions (they are not even defined on the same sample space).
%\s{this will likely confuse a lot of readers. consider giving an example, where say $\Xc = \{0,1\}$ and H contains some Gaussian or maybe even uniform CDFs}\sj{Too short on space, so might have to go with this. Tengyu suggests we submit to statistical learning theory track to get reviewers who can better understand these things.}

%In other words, for each fixed value $x \in \Xc$, $\fs[x]$ is a random variable that takes values in $\Fc(\Yc)$. 

If the $\fs$ takes some fixed value in $h \in \Hc$ with probability $1$, we call it a deterministic forecaster. Because deterministic forecasters are a subset of randomized forecasters, we will use $\fs$ to denote them as well.% unless otherwise mentioned. 
%\tnote{reworded slightly}

%We say that the forecaster is \textbf{deterministic} if $\forall x \in \Xc$, $F[x]$ is always the same element of $\Pc(\Yc)$, otherwise we say it is \textbf{stochastic}. 

\subsection{Background: Perfect Probability Forecast}

%\tnote{changed title.}
We consider several criteria that a good probability forecaster $\fs$ (randomized or deterministic) should satisfy, and whether they are achievable. %\tnote{does desiderata usefully mean a list of wanted things with different aspects. It seems that you have a list of criteria with increasing level}

% \tnote{I am not seeing why we have to introduce $\cdf^*_{XY}$. rephrased.}
% Suppose given a point $x$, the true conditional distribution $Y\mid X=x$ has CDF $\cdf^*_{Y \mid x}$, then an ideal probability forecaster %\tnote{ideal to idealistic?}
% %\s{not a native speaker, but i think it should be ideal.}
% should produce a CDF that is identical to the true conditional distribution. 
% \s{if you assume there is some kind of underlying joint distribution, then maybe mention it early on when you introduce x and y}

Given some input $x \in \Xc$, an ideal forecaster should always output the CDF of the true conditional distribution $\cdfyx$. We call such a forecaster a ``perfect forecaster''. %A perfect forecaster should satisfy $ \fs[x] = \cdf_{\ys \mid x}$, $\forall x \in \Xc$.
%Assume $x \in \Xc \subset \mathbb{R}^d$ while $y \in \mathbb{R}$.  
% \textbf{Perfect probability forecast} [Vovk et al]: A forecast $f$ is perfect, if no observer can distinguish between 
% $ f(x_1), y_1, \cdots $
% and the sequence $p^*(y|x_1), y_1, \cdots$. (e.g. by a Martingale test)
% \begin{definition}
% \label{def:perfect}
% A forecaster $\fs$ is perfect if $\forall x \in \Xc$, $ \fs[x] = \cdf_{\ys \mid x}$ almost surely. 
% % Let $p(y|x)$ be the true labeling function. Let $p(x, y) = p(x)p(y|x)$ denote their joint distribution. A forecaster is perfect if for all $x \in \Xc$, under the distribution $p(y|x)$ the random variable $f(x)(y)$ has distribution $\mathrm{Uniform}(0, 1)$;
% \end{definition}
% As a reminder, $\fs[x]$ is a random variable taking values in $\Fc(\Yc)$ (because $\fs$ is a randomized forecaster), on the other hand, $\cdf_{\ys \mid x}$ is a fixed CDF in $\Fc(\Yc)$. `Almost surely' refers to the random $\fs[x]$ taking the fixed value $\cdf_{\ys \mid x} \in \Fc(\Yc)$ with probability 1. 
%\s{should we say that H(x) is a random variable because the forecast is randomzied, while F is a fixed, deterministic CDF for a given x}

%This definition is arguably the strongest desideratum for a forecaster. 
However, learning an (approximately) perfect forecaster from training data is almost never possible. %Usually $\cdf^*$ is unknown to the forecaster, and only a training set of samples drawn from $\cdf^*$ are available. In fact, f
Usually each $x \in \Xc$ appear at most once in the training set (e.g. it is unlikely for the training set to contain identical images). It would be almost impossible to infer the entire CDF $\cdf_{Y \mid x}$ from a single sample $y \sim \cdf_{Y \mid x}$~\citep{vovk2005algorithmic} without strong assumptions. %\kh{Why does group calibration come after introducing individiual calibration? Shouldn't the order be flipped?} 

%\kh{I also found the asymmetry between the theory and intro a little strange. You introduce average, group, and individual calibration (in that order) in the intro. However the theory goes individual, group calibration.}

%\s{doesn't seem entirely true in general, with a strong enough prior (and you always need a prior..) you should be able to do it}
% In general it would be impossible to infer  from a single sample. In fact, perfect probability forecast is not even possible to achieve approximately.  \tnote{can shorten this paragraph, a detailed explanation does seem to be more useful than one sentence, and anothere sentence with the cite.}

% However, this is usually impossible to achieve with a deterministic forecaster, not even approximately. 
% \begin{theorem}
% To show that given finite data, there exists a typical distribution that could have generated the data, but $d_1( f(x, r)(y), U_0^1)$, where $d_1$ is the total variation distance, cannot be bounded.
% \end{theorem}
% The proof is almost identical to \citep{vovk2005algorithmic}. We exhibit two distributions that could have produced the same 

\subsection{Individual Calibration}

Because perfect probability forecasters are difficult to learn, we relax our requirement, and look at which desirable property of a perfect forecaster to emulate. 

%\tnote{mention somewhere globally that we are thinking about continuous $Y$}\sj{Mentioned in setup section}
We first observe that for some $x$, when the random variable $\ysx$ is truly drawn from a continuous CDF $\cdfyx \in \Fc(\Yc)$, by the inverse CDF theorem, $\cdfyx(\ysx)$ should be a random variable with uniform distribution in $[0, 1]$ --- As a  notation reminder, $\cdfyx$ is a fixed (CDF) function $\Yc \to [0, 1]$, $\ysx$ is a random variable taking values in $\Yc$. Therefore, $\cdfyx(\ysx)$ is a random variable taking values in $[0, 1]$. Also recall the convention that whenever $\ys$ appears in an expression, its distribution is always conditioned on $x$, i.e. $\ys \sim \cdfyx$. 
%\s{will people think Y is the marginal? here Y is sampled conditionally right?}

%\s{do we need some care to handle discrete vars / discontinuities?}\sj{Restricting to continuous CDFs now}
%U_0^1$ as the random variable uniformly distributed on $[0, 1]$. 

%This is the property we choose to emulate. 
If $\fs$ is indeed a perfect forecaster, then $\forall x \in \Xc$, $\fs[x]$ should always equal the true CDF: $\fs[x] = \cdfyx$. Therefore $\fs[x](\ysx)$ is a uniformly distributed random variable. In other words, $\fcdf$ is the CDF function of a uniform random variable. Conversely, we can require this property for any good forecaster. % to satisfy this property, which we call \textbf{individual calibration}. \tnote{rephrase this sentence.}
Formally, let $d: \Fc([0, 1]) \times \Fc([0, 1]) \to \mathbb{R}^+$ be any distance function (such as the Wasserstein-1 distance) between CDFs over $[0,1]$. For convenience we use $\ucdf$ to denote the CDF of a uniform random variable in $[0, 1]$. %With an abuse of notation, $F[x](Y)$ is a random variable taking values in $[0, 1]$, but we will also use it to denote its distribution.
We can measure 
\[ \err_\fs(x) \eqdef d(\fcdf, \ucdf) \]
and if $\err_\fs(x) = 0$ for all $x \in \Xc$, we say the forecaster $\fs$ satisfies individual calibration. %\s{i would explain again what is random, what is not, what these CDFs are}

 % approximate individual calibr which we define in the following definition.%In the following definition we will also quantify the error.
%For each $x$, $f(x)(y)$ is a random variable with randomness over $y$ and any randomization of $f$. Let $U_0^1$ be the uniform distribution on $[0, 1]$, then we can use $ d(f(x)(y), U_0^1) $ as a measurement of the accuracy of our probability forecast. 

\paragraph{Deterministic vs randomized forecasters.} 
Individual calibration for randomized forecasters has a weaker interpretation compared to deterministic forecasters. For example, the randomized forecaster that achieves individual calibration in Figure 1 has no deterministic counterpart. If one removes the randomness by averaging multiple random samples from $\fs[x]$, the individual calibration property is lost. 

In fact, if $\fs$ is deterministic (we denote it as $\fd$ instead), then $\err_\fd(x) = 0$ implies that $\fd[x]$ is the same CDF (almost everywhere) as $\cdfyx$. In other words, the forecaster has to produce exactly the correct distribution to satisfy individual calibration, which is limited by strong impossibility results in  \citep{barber2019limits,vovk2005algorithmic}. In fact, given finitely many data, we cannot even verify that a deterministic forecaster achieves individual calibration or not (Section 3). 

On the other hand, for most prediction tasks there always exists a randomized forecaster that trivially achieves individual calibration (Appendix~\ref{sec:appendix_trivial_mpaic}). Individual calibration along is not sufficient to guarantee that a randomized forecaster is useful, and must be supplemented with an additional "sharpness" objective. We will discuss this issue in Section 4.3. 

Even though individual calibration for randomized forecaster is a weaker definition, it has two key benefits. First, there is a sufficient condition to individual calibration that can be verified with a finite dataset (Section 4.1); second, with the additional "sharpness" objective,  it is useful for several real world applications (Section 5,6). 

%\tnote{say for deterministic models, this definition is almost equivalent to recovring $Y|X$ exactly. For randmoized it's not. Say explicitly the requirement of this condition for randomized model is much weaker --- give an example where you can trivially achieve individual calibration for randomized calibration.}

%\tnote{also add that if you define $\bar{H}$ as the average of $H$, then it stop satisfying the condition (per my reply to Omer's email)}\sj{addressed}

%\tnote{perhaps say that invidual calibration with randomized forecaster, though weeker than individual calibration with deterministic forecaster, still implies coverage.}\sj{Said so in the related work section}

%\tnote{also say that despite having a weaker calibration due to the randomization, still it's useful, point to section 5 and 6} 

%\paragraph{Measurability/Verifiability. } \tnote{say something about verifiability here if it makes sense} \tnote{say def 1 can be verifiabl if Y|X is deterministic. wee have a stronger version that is verifable in general in def 5. } 
%\tnote{mention that our work mostly apply to deterministic Y|X or nearly deterministic Y|X.}\sj{Said later in Section 4}

%\paragraph{Compare with other notions. }\tnote{compare with conditional coverage}\sj{Compared in related work section}

%\tnote{also cite all of those conformal prediction paper in relateed work. }

\paragraph{Approximate Individual Calibration.} In practice individual calibration can only be achieved approximately, i.e. $\err_\fs(x) \approx 0$ for most values of $x \in \Xc$, which we formalize in the following definition.

\begin{definition}[Individual calibration for randomized forecasters] %\tnote{edited}
\label{def:individual_calibration}
A forecaster $\fs$ is $(\epsilon,\delta)$-probably approximately individually calibrated (PAIC) (with respect to distance metric $d$) if %\tnote{make it $(\epsilon,\delta)$-probably ...}
\[ \Pr \left[ \err_\fs(\xs)\leq \epsilon \right] \geq 1- \delta \] 
\end{definition}
% Note that $\err$ is itself a function $\Xc \to \mathbb{R}$ that maps $x$ to the individual calibration error on that sample $x$. Then $\err(\xs)$ is a random variable. 
%\s{you probably should explain that now d is a random variable, because x is random.. what a mess}
% Naming: the naming of ``indistinguishable'' will become more sensible when we discuss an interpretation to PAI where a skeptic with no private knowledge, cannot distinguish between a PAI forecaster and a forecaster producing the ground truth conditional probabilities. 
%\s{marginal of x was never defined}

%\tnote{move this right after definition 1. and make it shorter}

%Individual calibration for deterministic forecaster is a stronger notion than for stochastic forecasters
%When the forecaster is deterministic, individual calibration in Definition~\ref{def:individual_calibration} implies that almost surely, the forecasted distribution must be identical to the true distribution --- a much stronger requirement than calibration for stochastic forecasters. \tnote{I don't think this is understandable for many readers.}
%However, as we show in the remaining of the paper, individual calibration for deterministic forecasters is stronger but unverifiable, while for stochastic forecasters is weaker but verifiable (thus can also be optimized with a learning algorithm).

This definition is intimately related to a standard definition of calibration for regression~\citep{gneiting2007probabilistic,kuleshov2018accurate}
which we restate slightly differently.

%\tnote{consider stating thee average calibiration befere individual calibration. Shall we have a subsection before this subsection for average calibration?}\sj{I did not take this advice because it's too big of a change to everything}
%\tnote{say that for average calibiration, randomized forecaster don't help.}
\begin{definition}[Average calibration]
\label{def:average_calibration} %\tnote{added the header}
A forecaster $\fs$ is $\epsilon$-approximately average calibrated (with respect to distance metric $d$) %\tnote{$\epsilon$-approximately..}
if 
\[ d(\fcdfs, \ucdf) \leq \epsilon \]
\end{definition}
%\s{call it average calibration maybe?}
Note that $d(\fcdfs, \ucdf) = 0$ is equivalent to the original definition of calibrated regression in \citep{kuleshov2018accurate} %  \s{which one?}
\begin{align*}
    \fcdfs(c) = \Pr[\fs[\xs](\ys) \leq c] = c = \ucdf(c), \forall c \in [0, 1]
\end{align*} 
% \s{
% \[
%  = \Pr[H[X](Y) \leq c] = c , \forall c \in [0, 1]
% \]
% }
In words, under the ground truth distribution, $y$ should be below the $c$-th quantile of the predicted CDF exactly $c$ percent of the times. More generally, an $\epsilon$-approximately average calibrated forecaster with respect to Wasserstein-1 distance also has $\epsilon$ ECE (expected calibration error) --- a metric commonly used to measure calibration error~\citep{guo2017calibration}. For details see  Appendix~\ref{sec:appendix_ece}. 
%error  a common measurement of calibration error -- ECE~\citep{guo2017calibration}  a $\epsilon$-approximately average calibrated forecaster exactly has expected calibration error (ECE) of $\epsilon$. We prove this in Appendix A. 

Despite the similarity, individual calibration is actually much stronger compared to average calibration (Figure~\ref{fig:comparison}). Individual calibration requires $\fs[x](\ys)$ be uniformly distributed for \textit{every} $x$. average calibration only require this \textit{on average}: $\fs[\xs](\ys)$ is uniformly distributed only if $\xs \sim \cdf_\xs$ --- it may not be uniformly distributed if $\xs$ has some other distribution. For example, if $\Xc$ can be partitioned based on gender, the forecaster can be uncalibrated when $\xs$ is restricted to a particular gender. On the other hand, average calibration is fairly easy to achieve with a deterministic forecaster~\citep{kuleshov2018accurate} --- there is no need or clear benefit from using a randomized forecaster. 
%be images of men / women only. 
%for the subset that only contain men, and the subset that only contain women. 
%\tnote{may need a bit more about this}
%\tnote{also give pointer to fiigure 2}
%\s{figure out a reasonable regression problem with possible fairness issues and use consistently throughout paper, intro and following examples. perhaps medical diagnosis. regress something from an image}

%If we select some subgroup $S \subset \Xc$, and have $\tilde{X} \sim \cdf^*_{X \mid X \in S}$, then $H[\tilde{X}](Y)$ may no longer be calibrated. 
% This requirement, is of course, difficult to achieve. In particular, we have the following impossibility result.

\subsection{Group Calibration}
To address the short-coming of average calibration in Definition~\ref{def:average_calibration}, a stronger notion of calibration has been proposed~\citep{kleinberg2016inherent, hebert2017calibration}. 
We choose in advance measurable subsets $\Sc_1, \cdots, \Sc_k \subset \Xc$ such that $\Pr[\xs \in \Sc_i] \neq 0, \forall i\in [k]$ (for \citet{hebert2017calibration} these are sets that can be identified by a small circuit from the input), and define the random variables $\xs_{\Sc_1}, \cdots, \xs_{\Sc_k}$ (Recall that $\xadvgi$ is distributed by $\xs$ conditioned on $\xs \in \Sc_i$). 
%\s{see above for conditioning comments}
%\tnote{we perhaps should give abbreviation for IC, GC, AGC, etc}
\begin{definition}[Group Calibration]%\tnote{cite here}
\label{def:group_calibration}
A forecaster $\fs$ is $\epsilon$-approximately group calibrated w.r.t. distance metric $d$ and $\Sc_1, \cdots, \Sc_k \subset \Xc$ if 
\[ \forall i\in [k], d\left(\cdf_{\fs[\xadvgi](\ys)}, \ucdf\right) \leq \epsilon \]
\end{definition}
%\s{again, doesn't look right}

This can alleviate some of the shortcomings of average calibration. However, the groups must be pre-specified or easy to compute from the input features $\xs$.  %When there are a large number of groups, there are no efficient algorithms that can verify group calibration~\citep{hebert2017calibration}.
A much stronger definition (Figure~\ref{fig:comparison}) is group calibration for any subset of $\Xc$ that is sufficiently large. 
\begin{definition}[Adversarial Group Calibration]
\label{def:adv_group_calibration}
A forecaster $\fs$ is $(\epsilon,\delta)$-adversarial group calibrated (with respect to distance metric $d$) if for $\forall \Sc \subset \Xc$ such that $\Pr[\xs \in \Sc] \geq \delta$ we have %for $\tilde{X} = X \mid \Sc$
\begin{align*} d\left(\cdf_{\fs[\xs_\Sc](\ys)}, \ucdf\right) \leq \epsilon \numberthis\label{eq:adv_group_calibration} \end{align*}
\end{definition}
%\s{maybe use $\tilde{X}(S)$ notation}
%Adversarial group calibration is weaker than individual calibration. It only requires $d\left(\cdf_{\fs[\xadvg](\ys)}, \ucdf\right) \approx 0$ for special types of random variables $\xadvg$ that are $\xs$ conditioned on $\xs$ belonging to a subset. However, it is stronger compared to group calibration because Eq.(\ref{eq:adv_group_calibration}) holds for any (sufficiently large) subset $\Sc$ instead of pre-specified subsets. 
% Remark: this is a much stronger notion that typical definitions of ``calibration''. We require calibration with respect to every subset. When $\delta=1$ we refer to this as \textbf{globally} calibrated, which is the usual notion of calibration~\citep{kuleshov2018accurate}.
%\s{also discuss relationship with group calibration}
% For detailed comparison of these definitions see 

 %, for deterministic forecasters, individual calibration is near-impossible to achieve. %The stochastic forecasters 
%the A stochastic forecaster can randomize the forecasted probability to satisfy the definition of calibration, 

%\section{Impossibility Results}
\newcommand{\yes}{{\mathrm{yes}}}
\newcommand{\no}{{\mathrm{no}}}

\section{Impossibility Results for Deterministic Forecasters}
We first present results showing that individual calibration and adversarial group calibration are impossible to verify with a deterministic forecaster. Therefore, there is no general method to train a classifier to achieve individual calibration. This motivates the need for randomized forecasters. 

%\kh{Is the impossibility result for group calibration not known? Is it important for purposes of this paper? I thought your focus was on individual calibration}

% \begin{prop}[Informal]\tnote{say "Informal version of Theorem ??"}
% \label{prop:impossible_deterministic}
% For any finite dataset $\lbrace (x_i, y_i) \rbrace$ where $x_i$ do not repeat, and any deterministic forecaster $\fd: \Xc \to \Fc(\Yc)$, there are two distributions $\cdf_1^*$ and $\cdf_2^*$ that equally plausibly generate the dataset. Under either $X, Y \sim \cdf_1^*$ or $X,Y \sim \cdf_2^*$
% \begin{align*} \Pr_{x}\left[ d(\cdf_{\fd[x](Y)}, \ucdf) \geq c \right] \geq 1/2 \numberthis\label{equ:error_upper_bound} \end{align*}
% % and under $p^*_2$
% % \[ \Pr_{x}\left[ d(f(x)(y), U_0^1) = 0 \right] = 1 \]
% If $d$ is the KL or Wasserstein-p distance, $c \geq 1/4$. 
% \tnote{why these two distribution equally plausible generate the dataset? Why equally plausible generation matters?}

% \end{prop}

Given a finite dataset $\Dc = \lbrace (x_i, y_i) \rbrace$ and a deterministic forecaster $\fd: \Xc \to \Fc(\Yc)$, suppose some verifier $T(\Dc, \fd) \to \lbrace \yes, \no \rbrace$ aims to verify if $h$ is $(\epsilon,\delta)$-PAIC. %``yes, the forecaster is $\epsilon,\delta$ PAIC''\tnote{make it $(\epsilon,\delta)$-PAIC}, or 'no the forecaster is not'. 
We claim that no verifier can correctly decide it (unless $\epsilon\ge 1/4$ or $\delta=1$ which means the calibration is trivially bad).  %whether the forecast is $\epsilon,\delta$ PAIC, as shown by the following proposition. 
This proposition is for the Wasserstein-1 distance $d_{W1}$. Examples of other distances are given in the Appendix~\ref{sec:appendix_impossibility}.

% First of all if $Y \mid x$ is a deterministic (i.e. $\cdf_{Y \mid x}$ is a step function) then individual calibration cannot be achieved because $\fs[x]$ is a continuous CDF. 

  %\tnote{rephrased, makes sense?}
%If individual calibration cannot be verified, even less can we design a learning objective and train a forecaster to satisfy it --- how does one design a learning objective to maximize a property that is unknown?

\begin{restatable}{prop}{impossibledeterministic}%[Informal]\tnote{say "Informal version of Theorem ??"}
\label{prop:impossible_deterministic}
For any distribution $\cdf$ on $\Xc \times \Yc$ such that $\cdf_\xs$ assigns zero measure to individual points $\lbrace x \in \Xc \rbrace$ sample $\Dc = \lbrace (x_1, y_1), \cdots, (x_n, y_n) \rbrace \simiid \cdf$. % and assume $x_i \neq x_j, \forall i \neq j$ almost surely.
 For any deterministic forecaster $h$ and any function
$ T(\Dc, \fd) \to \lbrace \yes, \no \rbrace $ %\tnote{$T$ to $T_{\epsilon,\delta}$? check other occurences.}\sj{Actually we don't need $\epsilon,\delta$, this is true for any function $T$}
such that 
\[ \Pr_{\Dc \sim \cdf}[T(\Dc, \fd) = \yes] = \kappa > 0\,, \]
there exists a distribution $\cdf'$ such that  (a) $\fd$ is not $\left( \epsilon, \delta \right)$-PAIC w.r.t $\cdf'$ for any $\epsilon < 1/4$ and $\delta < 1$, and (b)
\[ \Pr_{\Dc \sim \cdf'}[T(\Dc, \fd) = \yes] \geq \kappa \]
\end{restatable}

% The following theorem shows that individual calibration is \textbf{not verifiable} with finite data for deterministic forecasters. If we are given a forecaster $\fd$ and a finite validation set $\lbrace (x_i, y_i) \rbrace$ but no additional information, we can never prove a statement such as ``the forecaster $\fd$ is $(\epsilon, \delta)$-individually calibrated'' (unless $\epsilon\ge 1/4$ or $\delta=1$ which means the calibration is trivially bad). 
%This additionally implies that no learning algorithm can guarantee to produce an individual calibrated forecaster  
This additionally implies that no learning algorithm can guarantee to produce an individually calibrated forecaster. 
because otherwise the learning algorithm and its guarantee can serve as a verifier. 
%cannot select a forecaster that satisfies individual calibration from a set of hypothesis 
%, the learning algorithm should at least be able to verify that individual calibration. %optimize an objective that we cannot evaluate. 
Proofs of the proposition and a similar negative result for adversarial group calibration are in Appendix~\ref{sec:appendix_impossibility}.
% In words, for some forecaster $\fd$, if someone designs a verifier function $T$ for $(\epsilon,\delta)$ individual calibration, and claims that the verifier will say 'yes' with probability $\kappa$. %if there is some probability $\kappa$ the verifier says 'yes', 
% Then there must be an alternative distribution $\cdf'$ where the forecast makes a mistake with at least $\kappa$ probability. Because we do not know if the true distribution is $\cdf$ or $\cdf'$, we cannot rule out the worst possibility: whenever the verifier says 'yes' ($\kappa$ proportion of the times), it is always making a mistake. 
In addition to impossibility of verification, it is also known (in the conformal prediction setup) that no forecaster can guarantee individual calibration in a non-trivial way~\citep{barber2019limits}. For a discussion see related works. 

In light of these negative results, there are two options: one can make additional assumptions about the true distribution $\cdfjoint$, such as Lipschitzness. However, these assumptions are usually hard to verify, and their usefulness diminishes as the dimensionality of $\Xc$ increases. We propose an alternative: in contrast to deterministic forecasters, there is a sufficient condition for individual calibration for \emph{randomized} forecasters.
%\tnote{randomnized or stochastic?}. 
The sufficient condition is verifiable can be conveniently converted into a training objective. %In addition, when $\Xc$ is a high dimensional vector space, it is generally difficult to design assumptions that are both true and work 
% There are two solutions to this: one is to make additional assumptions about the underlying distribution; For example, in low dimensional problems it is common to assume that $p(y|x)$ is a Lipschitz function of $x$. However, this is not very helpful in high dimensions as the set of Lipschitz functions has exponential VC dimension. In general, is an important open problem in learning theory to design the appropriate assumptions and obtain practically useful guarantees.

\section{Individual Calibration with Randomized Forecasting} %\tnote{stochastic or randomized? be consistent}}
\subsection{Reparameterized Randomized Forecaster}

We will consider randomized forecasters that are deterministic functions applied on the input feature vector $x$ and some fixed random seeds $R$. (For readers familiar with variational Bayesian inference~\citep{kingma2013auto}, this is reminiscent of the reparameterization trick.)
Concretely, we choose a deterministic function $\fdr: \Xc \times [0, 1] \to \Fc(\Yc)$, and let $ \rs \sim \ucdf$.  Define the randomized forecaster $\fs[x]$ as %\tnote{rephrased}
\begin{align} \fs[x] = \fdr[x, \rs] \label{eqn:reparam}\end{align}

In addition, we would like $\fdr[x, r]$ to be a monotonic function of $r$ for all $x$. Any continuous function that is not monotonic  can be transformed into one by shuffling $r$.

%In general we can reparameterize $\fs$ with a higher dimensional random variable (i.e. $R$ is a uniform distribution on $[0, 1]^d$). This leads to a strictly more flexible set of forecasters because any additional random input can be ignored. However, a one dimensional $R$ is already sufficient.\tnote{it's unclear why it's sufficient, and it's sufficent for what. Do we really need it to claim this?}
% There is no loss of generality because if we can achieve individual calibration a smaller set of forecasters, we can certainly achieve it with a larger set of forecasters. 

\subsection{Sufficient Condition for Individual Calibration}
\label{sec:appendix_additional_discussion}

In this subsection, we introduce sufficient (but not necessary) conditions of individual calibration for %particular class for 
randomized forecasters defined in Eq.~\eqref{eqn:reparam}.
%The first step loses some generality. \tnote{is it implied that by sufficient condition that we will loss generality? }
First, recall that the definition of individual calibration requires $\fs[x](\ys) := \fdr[x, \rs](\ys)$ to be a uniform distribution for most of $x \in \Xc$ sampled from $\cdf_\xs$. This condition is hard to verify given samples, because for each $x$ in the training (or test) set, we typically only observe a single corresponding label $y \sim \cdf_{\ys \mid x}$. 

As an alternative, we propose to verify whether $\fdr[x, \rs](y)$ is uniformly distributed 
%\s{is this a distribution or a random variable that is uniformly distributed?} 
(for the unique sample $y$). Therefore we introduce a stronger but verifiable condition: %that will be easier to verify: 
\begin{align*}
\fdr[x, \rs](y) \text{ is uniformly distributed}  \numberthis\label{eq:weaker_cond}
\end{align*} 
for most (random) choices of $x, y$ under $\cdf_{XY}$.
% \tnote{technically you want to write $\fdr[x, R](y) \mid x,y$?}\sj{I think it's fine since x, y are just numbers not random variables by our notation.} \tnote{sure. (though technically as long as say $a\sim ?$ then $a$ should be random variable.)}
%This is the restrictive when $\cdf^*_{Y \mid x}$ is a distribution with high variance, but at least it is verifiable. 

%The second step does not lose generality. \tnote{it's not exactly clear what loss of generality means. Also why we cannot just directly talk about the stronger requirement? }

The benefit is that the condition --- $\fdr[x, \rs](y)$ is uniformly distributed --- can be written in an equivalent form when $\fdr[x,r]$ is a monotonic function in $r$ (proved in Theorem~\ref{thm:paicmono_imply_paic})%, which can in turn be verified: 
 %\tnote{"" move this sentence to earlier?}  %This is convenient because if $\fdr$ is monotonic in $r$, then $\fdr[x, R](y)$ is uniformly distributed in $[0, 1]$ if and only if
\[ \fdr[x, r](y) = r, \forall r \in [0, 1], \]
%\tnote{give a proof of this equivalence somewhere and give a pointer here to that proof?}
%Now we have reduced our original requirement %$d(\fdr[x, R](Y), \ucdf) \approx 0$ 
%into a sufficient condition that we can compute. 
%\s{doesn't type check? lhs is a CDF}
We formalize a relaxed version of this condition (allowing for approximation errors) as follows: %\tnote{shortened, pls check}
%not just $\forall x \in \Xc$, but also for any $y$ that is likely to be sampled by $\cdf^*_{Y \mid x}$. 

\begin{definition}
\label{def:paic_monotonic}
A forecaster $\fdr$ is $(\epsilon,\delta)$-monotonically probably approximately individually calibrated (mPAIC)   if
\[ \Pr \left[ \lvert \fdr[\xs, \rs](\ys) - \rs \rvert  \geq \epsilon\right] \leq \delta \]
\end{definition}
%\tnote{use macro for $U_0^1$? seems to have a new notation.}
%The key difference here is ``monotonically''. Definition~\ref{def:paic_monotonic} is a stronger requirement than Definition~\ref{def:individual_calibration}.
% \tnote{pernaps compare this with PAIC briefly?}\sj{I guess there is not much more I have in mind to say than the motivation of defining monotonic PAIC and the following theorem}

Note that even though we want to achieve $\fdr[x, r](y)=r, \forall r$, this does not mean that $\fdr$ ignores the input $x$  --- $\fdr[x, r](.)$ has the special form of a CDF, so $\fdr[x, r]$ must output a CDF concentrated around the observed label $y$ to satisfy mPAIC.

The following theorem formalizes our previous intuition that monotonic individual calibration (mPAIC) is obtained by imposing additional restrictions on individual calibration (PAIC) --- mPAIC is a sufficient condition for PAIC. 

% Here we force $f$ to use $r$ ``monotonically''. More formally, if $f$ is $\epsilon,\delta$ PIC by Definition~\ref{def:w-pai} an alternative forecaster $\tilde{f}$ defined by $\tilde{f}(x, r) = f(x, 1-r)$ is also $\epsilon,\delta$ PIC by Definition~\ref{def:pai} because the distribution of $\tilde{f}(x)$ is not changed. However, this does not satisfy Definition~\ref{def:w-pai}. In fact, for any volume preserving bijection $g: [0, 1] \to [0, 1]$, $\tilde{f}(x, r) = f(x, g(r))$ is still $\epsilon,\delta$ PIC by Definition~\ref{def:pai}. However, this reparameterization of $r$ is unimportant in practice --- we need to choose some method of using $r$, and we might as well choose to use it monotonically. This significantly simplifies our analysis and learning algorithm. From now on, we always use monotonically PIC, and we often drop the mention of monotonic.

\begin{restatable}{theorem}{paicmono}
\label{thm:paicmono_imply_paic}
If $\fdr$ is $(\epsilon, \delta)$-mPAIC, then for any $\epsilon' > \epsilon$ it is $(\epsilon', \delta(1-\epsilon)/(\epsilon'-\epsilon))$-PAIC with respect to the 1-Wasserstein distance. 
\end{restatable}
Proof can be found in Appendix~\ref{sec:appendix_mono_paic}.
This theorem shows that mPAIC implies PAIC (up to different constants $\epsilon,\delta$). In particular, if a forecaster achieves mPAIC perfectly (i.e. $\epsilon=0, \delta=0$), it is also perfectly PAIC.

The benefit of Definition~\ref{def:paic_monotonic} is that %it is a condition that can be analyzed under the standard empirical risk minimization framework. In other words, 
it can be verified with a finite number of validation samples, and is amenable to uniform convergence bounds, so that we can train it following the standard empirical risk minimization framework.  

\begin{restatable}{prop}{concentration}[Concentration]
\label{prop:concentration}
Let $\fdr$ be any $(\epsilon,\delta)$-mPAIC forecaster, and $(x_1, y_1), \cdots, (x_n, y_n) \simiid \cdfjoint$, $r_1, \cdots, r_n \simiid \ucdf$, then with probability $1-\gamma$%For any fixed $\hat{f}$ that is $\epsilon-\delta$ CDF calibrated, and data $(x_1, y_1), \cdots, (x_T, y_T) \sim P^*(x, y)$, then with probability $1-\gamma$, 
\[ \frac{1}{n} \sum_{i=1}^n \mathbb{I}(\lvert \fdr[x_i, r_i](y_i) - r_i \rvert \geq \epsilon) \leq \delta + \sqrt{\frac{-\log \gamma}{2n}}  \]
\end{restatable}

%\s{why would you use only one r per x,y pair?}\sj{To make the above expression look clean? Otherwise we need to have an additional expectation.}
%\tnote{$-r$ or $ -r_i$? $1/T$ to $1/n$? check other relevant parts for similar mismatch. }

% \begin{definition}
% Let $\hat{f}: \Xc \times \mathbb{R} \to \Fc$, we say that $f$ is $\epsilon-\delta$ CDF calibrated if under the measure $P^*(x, y)$ and $r \sim \mathrm{Uniform}[0, 1]$
% \[ P^*\left( \left\lvert \hat{f}(x, r)(y) - r \right\rvert \geq \epsilon \right) \leq \delta \]
% \end{definition}
% 

% \textbf{Training objective}
% Consider 
% \[ \min_{\theta} \frac{1}{T} \sum_i \left\lvert \hat{f}_\theta(x_i, r_i)(y_i) - r_i \right\rvert \]

\paragraph{Applicability of Sufficient Condition}

Our results are most useful when $\ys \mid x$ is almost deterministic, i.e. the uncertainty comes from model ignorance instead of the environment. When $\ys \mid x$ takes a distribution with high variance, Definition~\ref{def:paic_monotonic} can still be achieved, but with a significant sacrifice to sharpness. This is because $\fdr[x, r]$ must be a high variance distribution to satisfy $\fdr[x, r](\ys) \approx r$ for all likely values for $\ys \mid x$. %Even though individual calibration can still be (approximately) achieved, the sharpness may be prohibitively poor in practice. 

\subsection{Learning Calibrated Randomized Forecasters}

Given training data $(x_i, y_i) \simiid \cdfjoint$ we would like to learn a forecaster that satisfy mPAIC. We propose a concrete set of forecasters and a learning algorithm we will use in all of our experiments. 

We will model uncertainty with Gaussian distributions $\lbrace \Nc(\mu, \sigma^2), \mu \in \mathbb{R}, \sigma \in \mathbb{R}^+ \rbrace$. % ---  the forecaster $\fdr[x, r]$ outputs a CDF of a Gaussian distribution. 
We parameterize $\fdr = \fdr_\theta$ as a deep neural network with parameters $\theta$. The neural networks takes in  concatenation of $x$ and $r$ and outputs the $\mu, \sigma$ that decides the returned CDF.

% Consider the set of CDFs $\mathcal{N}(\mu, \sigma)$, and let $f: \Xc \times \mathbb{R} \to \mathbb{R}^2$ produce $\mu$ and $\sigma$. To parameterize $f$ we first define a deep feature extractor $g_\theta: \Xc \to \mathbb{R}^d$, and a linear function $f(x, r) = g_\theta(x)$.

Inspired by Proposition~\ref{prop:concentration}, we optimize the mPAIC objective on the training data defined as~\footnote{In practice we always sample a new $r_i$ for each training step. }
\[ \Lc_{\text{PAIC}}(\theta) = \frac{1}{n}\sum_{i=1}^n \lvert \fdr_\theta[x_i, r_i](y_i) - r_i \rvert  \]
Practically, this objective enforces the calibration but does not take into account the sharpness of the forecaster. For example, a simple $\fdr$ can be constructed that trivially outputs $r$ (Appendix~\ref{sec:appendix_trivial_mpaic}).
%\s{say it can be trivially optimized if h outputs r?}
We would  also like to minimize the variance $\sigma$ of the predicted Gaussian distribution. Therefore we also  regularize with the log likelihood (i.e. log derivative of the CDF) objective which encourages the sharpness of the prediction %\tnote{is this edit correct? also do we need to say why it encourages sharpness?}
\[ \Lc_{\text{NLL}}(\theta) = - \frac{1}{n} \sum_{i=1}^n \log \frac{d}{dy} \fdr[x_i, r_i](y_i)\]
Because we model uncertainty with Gaussian distributions, the $\Lc_{\text{NLL}}$ objective is equivalent to the standard squared error (MSE) objective typically used in regression~\citep{myers1990classical} literature.
We use a hyper-parameter $\alpha$ to trade-off between the two objectives:
\begin{align*}
    \mathcal{L}_\alpha(\theta) = (1- \alpha) \Lc_{\text{PAIC}}(\theta) + \alpha \Lc_{\text{NLL}}(\theta) \numberthis\label{eq:training_objective}
\end{align*}  
In other words, when $\alpha \approx 0$ the objective is almost exclusively PAIC, while when $\alpha = 1$ we reduce to the standard log likelihood maximization (i.e. MSE) objective.  %\s{mention standard regression loss, mse}

\section{Application I: Fairness}
% \tnote{should we add some overviewing texts here?}

\begin{figure*}
\begin{center}
\begin{tabular}{cc}
\includegraphics[width=0.94\linewidth]{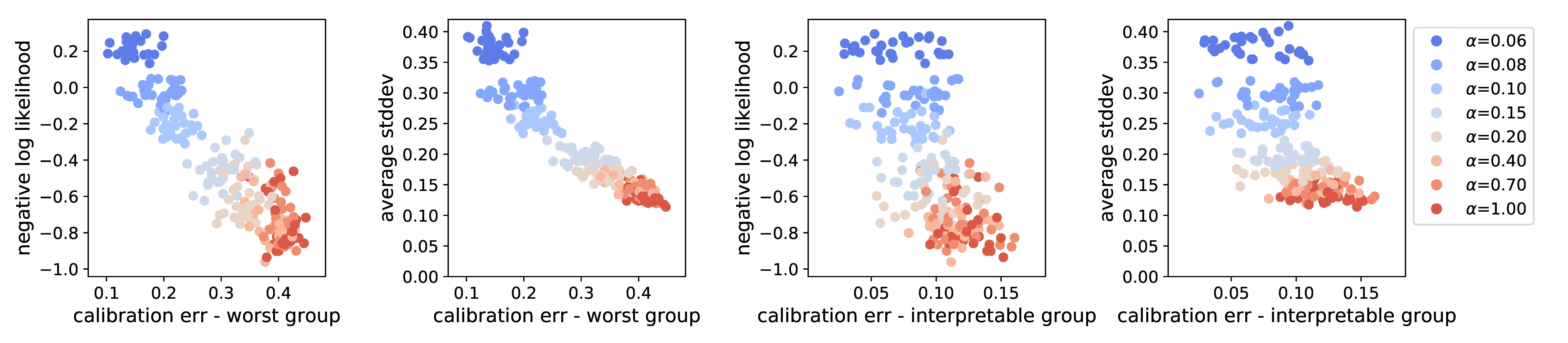} 
\end{tabular}
\vspace{-3mm}
\caption{Sharpness (negative log likelihood and average standard deviation) vs.
%\tengyu{vs, or against? should be a different proposition}
calibration error (on worst group and worst interpretable group) 
%\tnote{maybe also add a parenthesis here for worst-case group and interpretable group }
on the UCI crime dataset for different values of $\alpha$ (recall that $\alpha$ is the coefficient to trade-off $\Lc_{\text{NLL}}$ and $\Lc_{\text{PAIC}}$, $\alpha=1$ corresponds to standard training with $\Lc_{\text{NLL}}$). 
%\tnote{say $\alpha=0$ or 1? corresponds to the standard NLL training}
Each dot represent the performance of a classifier trained with some value of $\alpha$. It can be seen that there is a trade-off: a smaller value of $\alpha$ (more weight on $\mathcal{L}_{\mathrm{PAIC}}$ and less weight on $\mathcal{L}_{\mathrm{NLL}}$) leads to better calibration and worse sharpness. \textbf{Left 2}: Calibration loss on the worst group that contain 20\% of the test data. \textbf{Right 2}: Calibration loss on the worst interpretable group. 
%\tnote{in the xlable, can we use interpretable group instead of just interpretable?}
}
\label{fig:ece_crime}
\end{center}
\vspace{-3mm}
\end{figure*}

\begin{figure}
    \centering
    \includegraphics[width=1.0\linewidth]{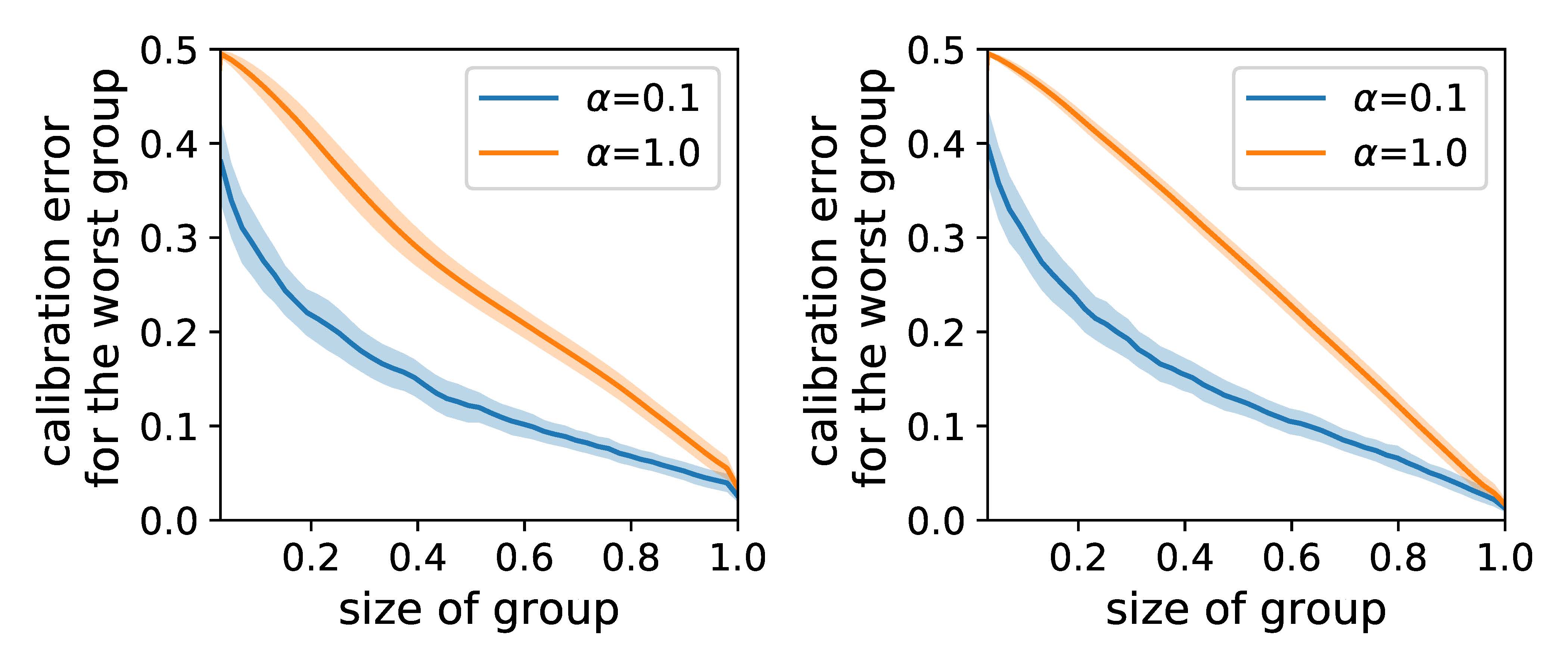}
    \caption{Calibration error as a function of the worst group size. The shaded region represents the one standard deviation error bar on random training/validation partitions. A forecaster trained with the PAIC objective ($\alpha=0.1$) has better calibration on adversarially chosen groups. \textbf{Left}: Without recalibration. \textbf{Right} With post training recalibration. Post training recalibration improves average calibration (i.e. group size = $1.0$) but does not improve adversarial group calibration (group size < $1.0$). % With recalibration, global calibration improves, but there is no guarantee that calibration will improve for groups. In fact, in this case, group calibration becomes worst on average.
    }
    \vspace{-5mm}
    \label{fig:ece_size}
\end{figure}

Individual calibration provides guarantees on the performance of a machine learning model on individual samples. As we will show, this has numerous applications. We begin discussing its use in settings where fairness is important. 
%We will look at how individual calibration can be applied in fairness. 
%We first show that individual calibration implies calibration for any group of data points, then we empirical verify that our method achieves better calibration on a crime prediction dataset.

\subsection{From Individual Calibration to Group Calibration}

In high-stakes applications of machine learning (e.g., healthcare and criminal justice), it is imperative that
%the social impact is a very important consideration. For example, 
predictions are fair. 
%and not discriminate against any group of the population. 
Many definitions of fairness are possible (see related work section), and calibration is a commonly used one.
%\tnote{"or at least necessary condition" do we need to claim this? might be a bit contraversial.}.
For example, in a healthcare application we would like to prevent a systematic overestimation or underestimation of a predicted risk for different socio-economic groups~\citep{pfohl2019creating}.
%\s{i think there is some work showing this for cardiovascular disease and race}
%\tnote{reprhased, does this work?} %can be unfairly treated, if we over predict the median crime rate (i.e. the true crime rate is below the predicted median more than 50\% of the times). % we would be over-predicting the median crime rate for the group. 
One natural requirement is that predictions for every group are calibrated, that is, the true value below the $r$\% quantile exactly $r$\% of the times. %for any group that we want to protect. 
%\s{note this is not guaranteed by average calibration: a model can be calibrated on average, while systematically underestimating risk in group, and overestimating the risk for another group.}

If the protected groups are known at training time, we could enforce group calibration as in Definition~\ref{def:group_calibration}.
However, it can be difficult to specify which groups to protect a priori. Some groups are also defined by features that are unobserved e.g. due to personal privacy. 
%and protecting exponentially many groups is computationally infeasible~\citep{hebert2017calibration}.\tnote{rephrased, ok?} 
%Calibration for the group $\lbrace x \in \Xc \mid x \text{ is green} \rbrace$ and calibration for the group $\lbrace x \in \Xc \mid x \text{ is tall} \rbrace$ does not imply calibration for $\lbrace x \in \Xc \mid x \text{ is green, tall and poor} \rbrace$. %Even for more sophisticated methods~\citep{hebert2017calibration}, 
%Explicitly labeling or parameterizing every protected group can be difficult~\citep{hebert2017calibration}. %\tengyu{also cite Omer's negative result?}\sj{Which one?}

We propose to address the problem by requiring a stronger notion of calibration, adversarial group calibration, where \emph{any group of sufficiently large size needs to be protected}. Moreover, we can achieve this stronger notion of calibration because it is implied by individual calibration. %the prediction is calibration for any group, which automatically include any protected group.
% In particular, it is sufficient to require individual calibration, because individual calibration implies adversarial group calibration according to the following theorem. 
% Remark: it is obvious that uniformly calibrated is a stronger notion than calibrated. In particular, if a forecaster is $\epsilon-\delta$ uniformly calibrated, then it is $\epsilon-\delta$ calibrated. 

% Strategy: show UPIC $\to$ uniformly calibrated. PIC $\to$ calibrated. PIC $\to$ uniformly calibrated, but the proof is very challenging and the resulting bound may be less tight. 

% We say that a probability forecast is $\epsilon,\delta$ group calibrated (with respect to distance $\Wc_s$), if for any subsequence selection function $g: x \to \lbrace 0, 1 \rbrace$, such that $Pr[g(x) = 1] > \delta$ 
% we say it is $\epsilon,\delta$ group calibrated if 
% \[  \left(\int_{c=0}^1 \left\lvert Pr[f(x)(y) < c|g(x)=1] - c \right\rvert^s \right)^{1/s} \leq \epsilon \]
% in particular, for $\Wc_\infty$ we have
% \[ c - \epsilon \leq Pr[f(x)(y) < c|g(x)=1] \leq c + \epsilon, \forall c \in [0, 1] \]

%Our main result is a correspondence between $\epsilon,\delta$ PIC and $\epsilon',\delta'$ individual calibration. 
\begin{restatable}{theorem}{advgroup}
\label{thm:pai_vs_calibration}
If a forecaster is $(\epsilon,\delta)$-PAIC with respect to distance metric $\Wc_p$, then $\forall \delta' \in [0, 1]$, $\delta' > \delta$, it is $(\epsilon + \delta/\delta', \delta')$-adversarial group calibrated with respect to $\Wc_p$.  
\end{restatable}
% \sj{TODO: also show a lower bound, exhibit a $\epsilon,\delta$ PIC forecaster but is not much better calibrated than the theorem claims. By the proof of the theorem it seems that this theorem is optimal up to a small constant.}
We prove a stronger version of this theorem in Appendix~\ref{sec:appendix_fairness}.
We know from theory that a forecaster trained on the individual calibration objective Eq.(\ref{eq:training_objective}) can achieve good individual calibration (and thus group calibration) on the training data. We will now experimentally verify that the benefit generalizes to test data.
%However, whether the benefit generalizes to test data usually cannot be shown theoreti  could achieve good individual calibration (and thus group calibration) on the training data, generalization could be We will not experimentally verify that learning on the individual calibration objective  achieve better calibration on any group compared to standard training objective of minimizing negative log likelihood. 

 %\tnote{say something about what we are experimenting either here or the beginning of the next subsection. }
 \subsection{Experiments}
 %\tnote{changed the hierachical structure if that's okay}
% 	\subsubsection{Experimental Setup}
% \textbf{Datasets.} 

\textbf{Experiment Details}. We use the UCI crime and communities dataset~\citep{dua2019uci} and we predict the crime rate based on features of the neighborhood (such as racial composition). For training details and network architecture see Appendix~\ref{sec:appendix_fairness}. \footnote{ https://github.com/ShengjiaZhao/Individual-Calibration}
%\s{explain what you arepredicting}

\textbf{Recalibration.} Post training recalibration is a common technique to improve calibration. %~\citep{guo2017calibration}. %\tnote{rephrased, okay?} 
%Post training recalibration can help improve standard calibrated, but has limited effect on individual or adversarial group calibration. 
%As shown in our experiments, 
%We will show that applying an additional post training recalibration does not change the conclusions of our experiments, but 
For completeness we additionally report all results when combined with recalibration by isotonic regression as in \citep{kuleshov2018accurate}. %\tnote{is there any evidence for "
Post training recalibration improves average calibration, but as we show in the experiments, has limited effect on individual or adversarial group calibration. %" If yes, we can say as shown in the experiment, ...  otherwise remove this sentence?} %Even though post-training calibration can successful improve global calibration. It does not appear to help in individual calibration. 

\subsubsection{Performance Metrics}
% The performance of the prediction can be measured by sharpness and calibration. 

\textbf{1) Sharpness metrics}: Sharpness measures whether the prediction $\fs[x]$ is concentrated around the ground truth label $y$. We will use two metrics: negative log likelihood on the test data $\Eb\left[ -\log \fdr(\xs, \rs)(\ys) \right]$, and expected standard deviation of the predicted CDFs
$ \Eb \left[ \sqrt{\Var[ \fdr(\xs, \rs)] } \right]$. Because forecaster outputs Gaussian distributions to represent uncertainty, this is simply the average standard deviation $\sigma$ predicted by the forecaster.  %\tnote{you meant $\Eb_X \left[ \sqrt{\Var[ \fdr(X, R)\mid X] } \right]$}

\textbf{2) Calibration metrics}: we will measure the ($\epsilon,\delta$)-adversarial group calibration (with 1-Wasserstein distance) defined in Definition~\ref{def:adv_group_calibration} and Eq.(\ref{eq:adv_group_calibration}). 
%\tnote{it would be good to point to a formal formula in early section}. 
In particular, we measure $\epsilon$ as a function of $\delta$ --- for a smaller group (smaller $\delta$) the $\epsilon$ should be larger (worse calibration), and vice versa. A better forecaster should have a smaller $\epsilon$ for any given value of $\delta$. We show in Appendix~\ref{sec:appendix_ece} that measuring $\epsilon$ is identical to the commonly used ECE~\citep{guo2017calibration} metric for miscalibration.
%The notion of $(\epsilon, \delta)$-adversarial group calibration also has a close connection to the commonly used ECE error. In fact, we will show in the Appendix that $d(\cdf_{\fs[\tilde{X}](Y)}, \ucdf)$ is equal to the ECE error. \tnote{not sure why we need to talk about ECE, also needs to give a def of ECE somewhere}
% \sj{Need to connect with previous defs}

The worst adversarial group may be an uninterpretable set, which is argurably less important than interpretable groups. Therefore, we also measure group calibration with respect to a set of known and interpretable groups. In particular, for each input feature we compute its the median value in the test data, and consider the groups that are above/below the median. For example, if the input feature is income, we considers the group with above median income, and the group with below median income. 
We also consider the intersection of any two groups. % for example, a group can be defined by: above median income, and below median percent of immigrants. %Compared to \citep{hebert2017calibration} these groups do not have to be known during training time. 
% two features, we can partition of data into four groups
% % Global ECE: measures the global calibration loss
% % \[ \int_{s=0}^1 \left\lvert \Pr_{x, y \sim \mathrm{test\ data}} [f(x, r)(y) \leq s] - s \right\rvert ds \]
% %     Note that an ECE error of 0.5 is the worst possible.
% % Worst case ECE at $\beta$ percent of the data, defined by choosing a subset of the test data with at least $\beta$ proportion of the data points. and compute the ECE on this subset. The adversary chooses the subset to maximize ECE error. 
% % Worst case interpretable ECE. Given any two features, we partition the data into four groups. 
%     \begin{center}
%     \begin{tabular}{cc|c|c}
%         & &  \multicolumn{2}{c}{Feature 1} \\
%         & &  $\geq$ median & $<$ median \\ \hline
%       \multirow{2}{0.18\linewidth}{Feature 2}  & $\geq$ median & Group 1 & Group 2 \\\cline{2-4}
%       &  $<$ median & Group 3 & Group 4
%     \end{tabular}
%     \end{center}
%\tnote{define the median of what? the median of the observed labels among the group with the feature?}
%We consider all such groups for all pairs of features, and find the group with the largest calibration error $\epsilon$. \tnote{do we have a name for this measure?}

\subsubsection{Results}

%\s{explain what the experiment is. summarize main findings, assuming reader won't look at figures}
The results are shown in Figure~\ref{fig:ece_crime} and \ref{fig:ece_size}. We compare different values of $\alpha$ (with $\alpha \approx 0$ we learn with $\Lc_{\text{PAIC}}$, and with $\alpha \approx 1$ we learn with $\Lc_{\text{NLL}}$). 

The main observation is forecasters learned with smaller $\alpha$ almost always achieve better group calibration (both adversarial group and interpretable group) and worse sharpness (log likelihood and variance of predicted distribution). %In particular, with small $\alpha$ (learning with PAIC objective), calibration for adversarial groups is significantly better, but the sharpness is worse. 
%This is to be expected because we optimize for a different objective other than log likelihood. 
This shows a trade-off between calibration and sharpness. %By choosing different values of $\alpha$, practitioners can choose the trade-off depending 
Depending on the application, a practitioner can adjust $\alpha$ and find the appropriate trade-off, e.g. given some constraint on fairness (maximum calibration error) achieve the best sharpness (log likelihood).
%In fact, \citeauthor{kleinberg2016inherent} and \citeauthor{pleiss2017fairness} prove (in the classification context) that a similar trade-off is inevitable. 

We also observe that post training recalibration improves average calibration (i.e. when the size of adversarial group is 100\% in Figure~\ref{fig:ece_size}). However, we cannot expect recalibration to improve individual or adversarial group calibration --- we empirically confirm this in Figure~\ref{fig:ece_size}. % it has limited effect on individual or adversarial group calibration, and  % and does not affect the conclusion of our experiments.  
%We provide additional experimental results with recalibration in Figure~\ref{fig:ece_crime2}, Appendix~\ref{sec:appendix_fairness}. 

In Table~\ref{table:groups} in Appendix~\ref{sec:appendix_fairness} we also report the worst interpretable groups that are miscalibrated. These do correspond to groups that we might want to protect, such as percent of immigrants, or racial composition.

\section{Application II: Decision under Uncertainty}
%\tnote{someway to narrow this down? decision making sounds too big: "Decision Making Under the Uncertainty"?}} 

%\subsection{Calibrated Forecasters for Decision Making}
Machine learning predictions are often used to make decisions. 
%We not only need accurate predictions but also need to have 
With good uncertainty quantification, the agent can consider different plausible outcomes, and pick the best action in expectation.
%to make the optimal decisions. 
%If the prediction is perfectly accurate, then usually an optimal decision can be made. 
%For example, %knowing perfectly the credit worthiness of a customer can help 
%an autonomous driving agent can make decisions on whether to decelerate based on the predicted distance to obstacles.
%Assessing the confidence in the predictions helps the trade. 
% a bank may make decision on accepting or denying a loan or credit card application based on the prediction of a customers credit worthiness and the uncertainty quantification of the prediction. Here the uncertainty quantification matters for the bank to control the risks. %However, because predictions are never perfectly accurate in practice, we must trade-off the potential loss of different outcomes that are plausible given the prediction. 
% \tnote{rephrased this a bit. Doesn't seem to connect well to uncertainty quantification. It's not only about accurate or not right? It's about whether we know the uncertainty. Perhaps need more revise}

More formally, suppose there is a set of actions $a \in \Ac$, %and the loss depends on the action $a$ and the input $x$ and the label $y \in \Yc$.
%\tnote{removed "true" just in case you want to have stochastic $y$}. 
and some loss function $l: \Xc \times \Yc \times \Ac \to \mathbb{R}$. If we had a perfect forecaster (i.e. $\fs[x] = \cdfyx$), then given input $x$, Bayesian decision theory would suggest to take the action that minimizes expected loss~\citep{fishburn1979two} under the predicted probability.
\begin{align*} 
    l_\fs(x) &\eqdef \min_a \Eb_{\tilde{\ys} \sim \fs[x]}[l(x, \tilde{\ys}, a)] \numberthis\label{eq:bayesian_loss} \\
    \phi_\fs(x) &\eqdef \arg\min_a \Eb_{\tilde{\ys} \sim \fs[x]}[l(x, \tilde{\ys}, a)] \numberthis\label{eq:bayesian_decision}
\end{align*} 
%\s{a reviewer might complain this is not bayesian if you don't integrate over H}
However, perfect forecaster is almost never possible in practice. %There must always be some mismatch between $\fs[x]$ and $\cdf^*_{Y \mid x}$. 
%Then how do we make decisions with an imperfect probability forecaster $\fs[x] \neq \cdfyx$?
% Given a forecast $\fs[x] \in \Pc(\Yc)$ we would like to make a decision $a$ with high utility. 
Nevertheless, calibration provides some guarantee on the decision rule in Eq.(\ref{eq:bayesian_decision}) for certain loss functions. 
%when the loss function has special properties. 
In particular, we consider loss functions $l(x, \cdot, a)$ that, for each $x, a$, are either monotonically non-increasing or non-decreasing in y. We call these loss functions \textbf{monotonic}. %This means that for each input $x$ and action $a$, there is a ``perferred'' direction of $y$ --- the loss is better either when $y$ is small, or when $y$ is large.
%Many loss functions have this property. 
For example, %in autonomous driving and $y$ is the distance to hitting an obstacle, then larger $y$  $\implies$ smaller loss.
if $y$ represents stock prices and $a \in \lbrace \text{buy}, \text{sell} \rbrace$, then when $a=\text{buy}$, loss is decreasing in $y$; when $a=\text{sell}$, loss is increasing in $y$. 

In the following theorem (proof in Appendix~\ref{sec:appendix_decision_making}) we show that the actual loss cannot exceed the expected loss in Eq.(\ref{eq:bayesian_loss}) too often. This would be Markov's inequality if Eq.(\ref{eq:bayesian_loss}) takes expectation under the true distribution $\cdfyx$. Interestingly the inequality is still true when the expectation is under the predicted distribution $\fs[x] \neq \cdfyx$. 
%Given an average calibrated forecaster (Definition~\ref{def:average_calibration}), we will know that certain values of $x, y$ are unlikely. More specifically, the set $\Sc_{\mathrm{bad}} = \lbrace (x, y) \in \Xc \times \Yc \mid \fs[x](y) < c \rbrace$ should have probability measure at most $c$ under the true distribution $\cdfjoint$. If we choose some small value $c$, then we know that $(x, y) \in \Sc_{\mathrm{bad}}$ is unlikely to happen (when $x, y \sim \cdfjoint$).

\begin{restatable}{theorem}{decision}
\label{thm:markov}
Suppose $l: \Xc\times \Yc \times \Ac \to \mathbb{R}$ is a monotonic non-negative loss, let $\phi_\fs$ and $l_\fs$ be defined as in Eq.(\ref{eq:bayesian_decision})

1. If $\fs$ is $0$-average calibrated, then $\forall k > 0$
\begin{align*}
    \Pr[l(\xs, \ys, \phi_\fs(\xs)) \geq k l_\fs(\xs)] \leq 2/k
\end{align*}
% 2. If $\fs$ is $(0, 1/2)$-adversarial group calibrated
% \begin{align*}
%     \Pr[l(\xs, \ys, \phi_\fs(\xs)) \geq k l_\fs(\xs)] \leq 1/k
% \end{align*}
2. If $\fs$ is $(0, 0)$-PAIC, then $\forall x \in \Xc, k > 0$
\begin{align*}
    \Pr[l(x, \ys, \phi_\fs(x)) \geq k l_\fs(x)] \leq 1/k
\end{align*}
\end{restatable}

\subsection{Case Study: Credit Prediction}

\begin{figure}
    \centering
    %\adjincludegraphics[height=5cm,trim={0 0 {.5\width} 0},clip]{example-image-a}
    \includegraphics[width=1.0\linewidth]{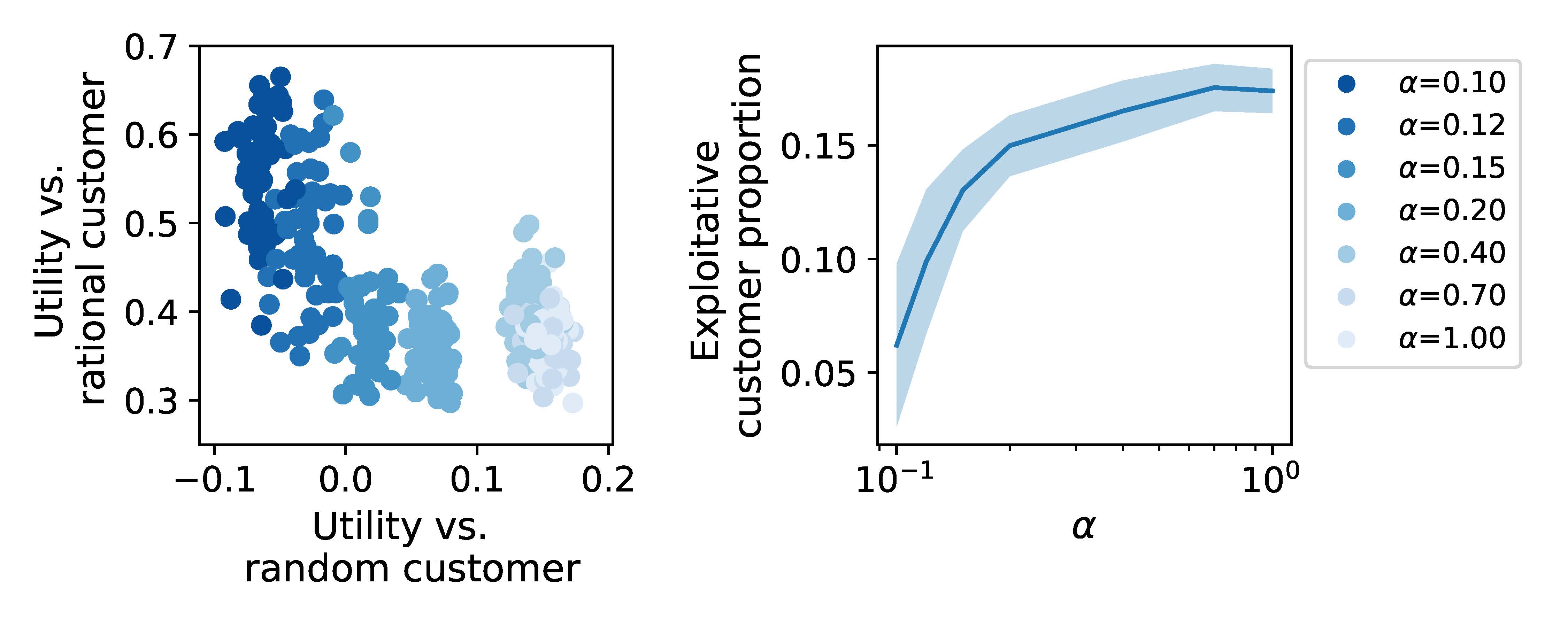}
    \caption{Comparison between individual calibration ($\alpha \approx 0$) and baseline ($\alpha = 1$). \textbf{Left}: Each dot represents a random roll-out for different values of $\alpha$, where we plot the bank's average utility when the customers either decide to apply randomly or rationally. Individually calibrated forecaster perform worse than baseline when the customers are random, but better when customers are rational. \textbf{Middle Left}: Proportion of exploitative customers (customers with $y < y_0$ but decide to apply). Random forecasters have less systematic bias and discourages exploitative customers.  
    %\textbf{Right 2}: The same experiment, but with post training recalibration on the bank's forecaster. \s{what's the takeway of the right vs left plots?} 
    }
    \label{fig:credit_simulation}
\end{figure}

%Consider a simplified model of a financial decision. 
Suppose customers of a financial institution are represented with a feature vector $x$ and a real-valued credit worthiness score $y \in \Yc$. The bank has a financial product (e.g. credit card or loan) with a minimum threshold $y_0 \in \mathbb{R}$ for credit worthiness. 
If a customer chooses to apply for the product, the bank observes $x$, and uses forecaster $\fs$ to predict %given a customer's financial feature $x$, 
 their true credit worthiness $y$. 
 %The bank only wants to say `yes' if the credit worthiness $y \geq y_0$. % and say yes or no. 
There is a positive utility for saying 'yes' to a qualified customer, and a negative utility for saying 'yes' to a disqualified customer. More specifically, the utility (negative loss) for the bank is
\begin{center}
\begin{tabular}{c|cc}
   & $y \geq y_0$ & $y < y_0$ \\ \hline
    'yes' & 1 & -3 \\
    'no' & 0 & 0
\end{tabular}
\end{center}
%If the bank has a average calibrated forecaster $\fs$ and computes the set $\Sc_{\mathrm{bad}} = \lbrace (x, y) \in \Xc \times \Yc \mid \fs[x](y) < c \rbrace$. Suppose for a customer $x$, $(x, y) \in \Sc_{\mathrm{bad}}$ whenever $y < y_0$, then the bank will know that it is unlikely that the true $y$ is $< y_0$. This corresponds to a decision rule

\paragraph{Guarantees from Average Calibration:} Suppose the bank uses the Bayesian decision rule in Eq.(\ref{eq:bayesian_decision}). If $\fs$ is average calibrated, 
%to apply Theorem~\ref{thm:markov} we define the loss as the negative utility + 1. 
Theorem~\ref{thm:markov} would imply that when the bank says 'yes', at most $25\%$ will be to unqualified customers (when customers truly come from the distribution $\cdf_\xs$). For details see Appendix~\ref{sec:appendix_credit_guarantee}.
%This is a non-negative loss that satisfies the condition of Theorem~\ref{thm:markov}. Whenever the bank says 'yes', it means that $\phi_\fs(x) \geq 1$, 

% If $\fs$ is indeed average calibrated (i.e. $\fs[X](Y)$ is a uniform distribution), and assume the customers that apply truly come from the distribution $\cdf^*$, then the decision rule in Eq.(\ref{eq:bank_strategy}) will say 'yes' to unqualified customer at most $c$ proportion of the times. 

%\subsubsection{Playing against Rational Customers}
The Bayesian decision rule in Eq.(\ref{eq:bayesian_decision}) is fragile when $\fs$ is only average calibrated %when the customers know the decision rule of the bank and are rationally maximizing their utility. This is 
because the guarantee above is void if the customers do not come from the distribution $\cdf_\xs$. 
For example, suppose some unqualified customers know that their credit scores are overestimated by the bank,
%(such customer might exist if forecaster is only average calibrated), 
then these customers are more likely to apply. % and therefore the distribution of the customer will shift from ??. \tnote{notation}
%Unqualified customers who are mistakenly approved could have the highest utility --- there could be up to $c$ proportion such customers --- and more likely decide to apply. 
% \tnote{I don't see why we need this utility table from customer. }
%If customers are perfectly rational, they will try to exploit the bank, and more likely apply if applying leads to a higher utility for themselves. 
More concretely, if only the customers who would be mistakenly approved ended up applying, then the bank is guaranteed to lose (it will suffer a loss of -3 for each customer). 

\noindent{\bf Guarantees from Individual Calibration:} If $\fs$ is individually calibrated (and hence also calibrated with respect to any adversarial subgroup), the bank cannot be exploited --- no matter which subgroup of customers choose to apply, the fraction of unqualified approvals is at most $25\%$.   %Customers will adapt to the bank's strategy and try to maximize their own utility. The utility for the customer is 

% TODO: analyze as more users apply ($\psi$ is learned better) the bank's utility degrades more without individual calibration. 
% and exploit any weakness the bank has due to its lack of information about the customer. 

\subsubsection{Simulation}
We perform simulations 
to verify that individual calibrated forecasters are less exploitable and can achieve higher utility in practice.
%--- when there is finite training data and approximate calibration. 
We model the customers as rational agents that predict their utility and exploit any mistake by the bank. For detailed setup, see Appendix~\ref{sec:appendix_credit_details}. 

The results are shown in Figure~\ref{fig:credit_simulation}. We compare different values of $\alpha \in [0.1, 1.0]$ (recall when $\alpha \approx 0$ we almost exclusive optimize $\Lc_{\text{PAIC}}$, and when $\alpha = 1$ we exclusively optimize $\Lc_{\text{NLL}}$). We compare the average utility on random customers (customers drawn from $\cdfjoint$) and rational customers (Appendix~\ref{sec:appendix_credit_details}). % (customers who apply if $\psi(x, y) \geq 0$).  

The main observations is that when customers are random, individual calibrated forecasters perform worse because average calibration is already sufficient. %The additional randomness introduced by individual calibration actually increases the probability of decision error. 
On the other hand, when the customers are rational, and try to exploit any systematic bias with the decision maker, individually calibrated forecasters perform much better. %Experimental results without post training recalibration are qualitatively similar and presented in Figure~\ref{fig:credit_simulation2} in Appendix~\ref{sec:appendix_credit_details}.

% Note that the bank actually has higher utility in general when the customers are rational. \tnote{this is only for the I.C. case, right? mention it.}This is because the game is partly cooperative --- most unqualified customers ($y < y_0$) will not apply, this increases the proportion of qualified customers among those who apply. 
% \input{safety.tex}

%\input{bandit.tex}

\section{Related Work}

\textbf{Randomized Forecast}:
Randomized forecast has been used in adversarial environments. In online learning where the true label can be adversarial, randomized forecasters can achieve low regret~\citep{cesa2006prediction,shalev2012online}. In security games / Stackelberg games~\citep{tsai2010urban, trejo2015stackelberg}, % defender makes a forecast (e.g. which airport needs additional security) and attacker chooses a strategy knowing the defender strategy. 
defender needs randomization to play against attackers. % to successfully defend against attacks. % where defender announces the strategy, and attacker choose strategy based on defender strategy, 
%Compared to prior work, we use randomized forecasters to achieve a new type of guarantee. % calibration even for groups selected by an adversary. 
\citep{perdomo2020performative} gives a theoretical characterization of prediction in (possibly adversarial) environments when the loss function is strongly convex.  

%An alternative to randomization is conformal prediction. Instead of forecasting a single probability, conformal prediction forecaster output a set of probabilities~\citep{vovk2005algorithmic,shafer2008tutorial}.
%and the forecast is considered successful if at least one of these probabilities match the true probability. 
%It is interesting future work to explore this alternative framework in fairness and decision making. %Conformal prediction avoids randomization, but also can be very slow and difficult to implement %Conformal prediction avoid randomization, but are often more difficult to use in practice because of the need to represent a set of probabilities.  %These methods methods have been applied 
% proposed probability forecast, and proved the impossibility of perfect probability forecast. In addition, \citep{vovk2005algorithmic} proposed an alternative method (Venn prediction) to achieve perfect probability forecast: by outputting multiple probability distributions. However, it is not clear how to scale up these methods. (Actually I'm not sure it is correct)

\textbf{Calibration}:
Definitions of average calibration first appeared in the statistics literature \citep{brier1950verification, murphy1973new, dawid1984present,cesa2006prediction}. Recently there has been a surge of interest in recalibrating classifiers~\citep{platt1999probabilistic, zadrozny2001obtaining, zadrozny2002transforming, niculescu2005predicting}, especially deep networks~\citep{guo2017calibration, lakshminarayanan2017simple}. Average calibration in the regression setup has been studied by \citep{gneiting2007probabilistic,kuleshov2018accurate}.

Group calibration for a small number of pre-specified groups has been studied in~\citep{kleinberg2016inherent}.  Interestingly \citep{hebert2017calibration,kearns2017preventing} can achieve calibration for any group computable by a small circuit, but is computationally difficult (likely no polynomial time algorithm). Similarly \citep{barber2019limits} achieve calibration for a set of groups, but only has efficient algorithms for special sets of groups.
\citep{kearns2019average} proposes a notion of individual calibration applicable when there are many prediction tasks, and each individual draws multiple prediction tasks. 
%. i.i.d. multiple prediction tasks.  the ``average error'' for the individual can be meaningfully defined. Our approach do not require multiple prediction tasks and thus more general.
\citep{joseph2016fairness} achieves a notion similar to individual calibration for fairness, but needs strong realizability assumptions that are difficult to verify.  %--- individuals with smaller label (i.e. credit) should not be selected with higher probability. However, the algorithm 
\citep{liu2018implicit} proves an upper bound on calibration error for any group. However, it is unclear how to compute the upper bound if the group labels are not provided.

%This impossibility results is complementary to our result: \citep{barber2019limits} show that individually calibration cannot be (non-trivially) achieved with guarantees, and we show that it cannot be verified even if achieved in practice. 
% achieves a type of individual calibration with stochastic forecasters. It formulates the problem as a contextual bandit, where the different classification labels correspond to different arms, and properties of an individual correspond to the context. The algorithm would like to guarantee that for each context, the probability of pulling an arm (label) with higher reward is always higher than pulling an arm with lower reward. This is a very strong notion: if one were able to achieve this, it would imply $\epsilon-0$ perfect forecast under ignorance. Of course as we proved that it is not achievable without assumptions or repeated observations. Indeed \citep{joseph2016fairness} assumes either: for each context we observe multiple samples, or make assumptions about the function class (it is KWIK learnable). The former assumption is usually not true, while the latter assumption is difficult to verify. 

\textbf{Calibration for Regression} We extend the definition of regression calibration proposed in \citep{kuleshov2018accurate}. This is not the unique reasonable definition for regression calibration. For example, we can partition $\Yc$ into bins to convert the regression problem into a classification problem, apply the classification calibration definition~\citep{vovk2005algorithmic,guo2017calibration}, and take the number of bins to infinity. This leads to a different definition for calibration than ours (even in the limit of infinitely many bins). Alternative definitions have also been proposed in~\citep{levi2019evaluating}. Individual calibration that extends these alternative definitions is beyond the scope of this paper.

\textbf{Conformal Prediction} Individual guarantees are also discussed in the conformal prediction literature \citep{vovk2005algorithmic}, where the objective is to produce an set that contains the true label with guaranteed high probability. \citep{vovk2012conditional,barber2019limits} show that if a conformal predictor \textit{always} satisfies conditional coverage (i.e. the true label belong to the predicted set with the advertised probability for each individual sample), then the size of the set must be impractically large. Therefore \citep{barber2019limits} instead achieves a notion of coverage for groups. 

An individually calibrated forecaster can be converted into a conformal predictor that satisfies conditional coverage --- $\ys \mid x$ belongs to the set $[\fs[x]^{-1}(\alpha/2), \fs[x]^{-1}(1-\alpha/2)]$ with $1-\alpha$ probability if the forecaster is individually calibrated.
However, our algorithm is not bound by the impossibility result of \citep{barber2019limits,vovk2012conditional} because our algorithm do not \textit{always} produce individually calibrated forecasters --- success of the algorithm depends on the inductive bias and the data distribution. However, whenever the algorithm succeeds in producing individually calibrated forecasters, we do obtain post-training guarantees by Theorem~\ref{prop:concentration}.  
%\citep{vovk2012conditional,barber2019limits} show that individual coverage is also (almost) impossible, and \citep{barber2019limits} settles for the weaker condition of coverage for groups. 

\textbf{Fairness}:
% \citep{cesa2006prediction} provides a overview of calibration without i.i.d. assumptions about the underlying distribution. For example, it could be generated by an adversary. However, under such a setup, only asymptotic guarantees are possible, and practical algorithms are rare. More recent examples include \citep{kuleshov2017estimating}
% \subsection{Fairness}
% The notion of calibration is closely related to fairness. It is equivalent to calibration under checking rules. Group fairness is identical to using a small number of checking rules. These papers include
In addition to calibration~\citep{kleinberg2016inherent,hebert2017calibration,kearns2017preventing}, other definitions of fairness include metric based fairness~\citep{dwork2012fairness}, equalized odds~\citep{hardt2016equality}, conterfactual fairness~\citep{kusner2017counterfactual, kilbertus2017avoiding}, and representation fairness~\citep{zemel2013learning,louizos2015variational,song2018learning}. The trade-off between these definitions are discussed in~\citep{pleiss2017fairness,kleinberg2016inherent,friedler2016possibility,corbett2017algorithmic}. %Future work can explore whether randomization can help with alternative notions of fairness. %We leave application of our results to alternative definitions of fairness as future work.

\section{Conclusion and Future Work}
In this paper we explore using randomization to achieve individual calibration for regression. We show that these individually calibrated predictions are useful for fairness or decision making under uncertainty. 
One future direction is extending our results to classification. 
%The definition of calibration for regression and classification are very different, and 
The challenge is that there is no natural way to define a CDF for a discrete random variables. % (unless we unnaturally embed the discrete labels on the real line). %One possibility is to predict the mean of the discrete random variable as a regression problem. %We leave this for future work.
Another open question is a good theoretical characterization of the trade-off between sharpness and individual calibration. % that explain the trade-off we observe in all of our experiments. 

% \section{Conclusion}
% In this paper we explore the use the randomized forecaster to achieve individual calibration. We show that individually calibrated forecasters can have an advantage in fairness and decision making. 
% I think we need some discussion about whether this approach is applicable to classification.
% \section{Acknowledgements}

% This research was supported by TRI, Intel, NSF (\#1651565, \#1522054, \#1733686), and ONR.

%\input{forecast_knowledge.tex}

%\input{old2.tex}

\section{Acknowledgements}
This research was supported by AFOSR (FA9550-19-1-0024), NSF (\#1651565, \#1522054, \#1733686), JP Morgan, ONR, TRI, FLI, SDSI, and SAIL. 
Toyota Research Institute ("TRI")  provided funds to assist the authors with their research but this article solely reflects the opinions and conclusions of its authors and not TRI or any other Toyota entity.
TM is also supported in part by Lam Research and Google Faculty Award. 

We are thankful for valuable feedback from Kunho Kim (Stanford), Ananya Kumar (Stanford) and Wei Chen (MSRA). 
\bibliographystyle{icml2020}
\bibliography{main}

\begin{thebibliography}{42}
\providecommand{\natexlab}[1]{#1}
\providecommand{\url}[1]{\texttt{#1}}
\expandafter\ifx\csname urlstyle\endcsname\relax
  \providecommand{\doi}[1]{doi: #1}\else
  \providecommand{\doi}{doi: \begingroup \urlstyle{rm}\Url}\fi

\bibitem[Barber et~al.(2019)Barber, Candes, Ramdas, and
  Tibshirani]{barber2019limits}
Barber, R.~F., Candes, E.~J., Ramdas, A., and Tibshirani, R.~J.
\newblock The limits of distribution-free conditional predictive inference.
\newblock \emph{arXiv preprint arXiv:1903.04684}, 2019.

\bibitem[Brier(1950)]{brier1950verification}
Brier, G.~W.
\newblock Verification of forecasts expressed in terms of probability.
\newblock \emph{Monthly weather review}, 78\penalty0 (1):\penalty0 1--3, 1950.

\bibitem[Cesa-Bianchi \& Lugosi(2006)Cesa-Bianchi and
  Lugosi]{cesa2006prediction}
Cesa-Bianchi, N. and Lugosi, G.
\newblock \emph{Prediction, learning, and games}.
\newblock Cambridge university press, 2006.

\bibitem[Corbett-Davies et~al.(2017)Corbett-Davies, Pierson, Feller, Goel, and
  Huq]{corbett2017algorithmic}
Corbett-Davies, S., Pierson, E., Feller, A., Goel, S., and Huq, A.
\newblock Algorithmic decision making and the cost of fairness.
\newblock In \emph{Proceedings of the 23rd ACM SIGKDD International Conference
  on Knowledge Discovery and Data Mining}, pp.\  797--806. ACM, 2017.

\bibitem[Dawid(1984)]{dawid1984present}
Dawid, A.~P.
\newblock Present position and potential developments: Some personal views
  statistical theory the prequential approach.
\newblock \emph{Journal of the Royal Statistical Society: Series A (General)},
  147\penalty0 (2):\penalty0 278--290, 1984.

\bibitem[Dua \& Graff(2017)Dua and Graff]{dua2019uci}
Dua, D. and Graff, C.
\newblock {UCI} machine learning repository, 2017.
\newblock URL \url{http://archive.ics.uci.edu/ml}.

\bibitem[Dwork et~al.(2012)Dwork, Hardt, Pitassi, Reingold, and
  Zemel]{dwork2012fairness}
Dwork, C., Hardt, M., Pitassi, T., Reingold, O., and Zemel, R.
\newblock Fairness through awareness.
\newblock In \emph{Proceedings of the 3rd innovations in theoretical computer
  science conference}, pp.\  214--226. ACM, 2012.

\bibitem[Fishburn \& Kochenberger(1979)Fishburn and
  Kochenberger]{fishburn1979two}
Fishburn, P.~C. and Kochenberger, G.~A.
\newblock Two-piece von neumann-morgenstern utility functions.
\newblock \emph{Decision Sciences}, 10\penalty0 (4):\penalty0 503--518, 1979.

\bibitem[Friedler et~al.(2016)Friedler, Scheidegger, and
  Venkatasubramanian]{friedler2016possibility}
Friedler, S.~A., Scheidegger, C., and Venkatasubramanian, S.
\newblock On the (im) possibility of fairness.
\newblock \emph{arXiv preprint arXiv:1609.07236}, 2016.

\bibitem[Gneiting et~al.(2007)Gneiting, Balabdaoui, and
  Raftery]{gneiting2007probabilistic}
Gneiting, T., Balabdaoui, F., and Raftery, A.~E.
\newblock Probabilistic forecasts, calibration and sharpness.
\newblock \emph{Journal of the Royal Statistical Society: Series B (Statistical
  Methodology)}, 69\penalty0 (2):\penalty0 243--268, 2007.

\bibitem[Guo et~al.(2017)Guo, Pleiss, Sun, and Weinberger]{guo2017calibration}
Guo, C., Pleiss, G., Sun, Y., and Weinberger, K.~Q.
\newblock On calibration of modern neural networks.
\newblock In \emph{Proceedings of the 34th International Conference on Machine
  Learning-Volume 70}, pp.\  1321--1330. JMLR. org, 2017.

\bibitem[Hardt et~al.(2016)Hardt, Price, Srebro, et~al.]{hardt2016equality}
Hardt, M., Price, E., Srebro, N., et~al.
\newblock Equality of opportunity in supervised learning.
\newblock In \emph{Advances in neural information processing systems}, pp.\
  3315--3323, 2016.

\bibitem[H{\'e}bert-Johnson et~al.(2017)H{\'e}bert-Johnson, Kim, Reingold, and
  Rothblum]{hebert2017calibration}
H{\'e}bert-Johnson, U., Kim, M.~P., Reingold, O., and Rothblum, G.~N.
\newblock Calibration for the (computationally-identifiable) masses.
\newblock \emph{arXiv preprint arXiv:1711.08513}, 2017.

\bibitem[Joseph et~al.(2016)Joseph, Kearns, Morgenstern, and
  Roth]{joseph2016fairness}
Joseph, M., Kearns, M., Morgenstern, J.~H., and Roth, A.
\newblock Fairness in learning: Classic and contextual bandits.
\newblock In \emph{Advances in Neural Information Processing Systems}, pp.\
  325--333, 2016.

\bibitem[Kearns et~al.(2017)Kearns, Neel, Roth, and Wu]{kearns2017preventing}
Kearns, M., Neel, S., Roth, A., and Wu, Z.~S.
\newblock Preventing fairness gerrymandering: Auditing and learning for
  subgroup fairness.
\newblock \emph{arXiv preprint arXiv:1711.05144}, 2017.

\bibitem[Kearns et~al.(2019)Kearns, Roth, and
  Sharifi-Malvajerdi]{kearns2019average}
Kearns, M., Roth, A., and Sharifi-Malvajerdi, S.
\newblock Average individual fairness: Algorithms, generalization and
  experiments.
\newblock \emph{arXiv preprint arXiv:1905.10607}, 2019.

\bibitem[Kilbertus et~al.(2017)Kilbertus, Carulla, Parascandolo, Hardt,
  Janzing, and Sch{\"o}lkopf]{kilbertus2017avoiding}
Kilbertus, N., Carulla, M.~R., Parascandolo, G., Hardt, M., Janzing, D., and
  Sch{\"o}lkopf, B.
\newblock Avoiding discrimination through causal reasoning.
\newblock In \emph{Advances in Neural Information Processing Systems}, pp.\
  656--666, 2017.

\bibitem[Kingma \& Welling(2013)Kingma and Welling]{kingma2013auto}
Kingma, D.~P. and Welling, M.
\newblock Auto-encoding variational bayes.
\newblock \emph{arXiv preprint arXiv:1312.6114}, 2013.

\bibitem[Kleinberg et~al.(2016)Kleinberg, Mullainathan, and
  Raghavan]{kleinberg2016inherent}
Kleinberg, J., Mullainathan, S., and Raghavan, M.
\newblock Inherent trade-offs in the fair determination of risk scores.
\newblock \emph{arXiv preprint arXiv:1609.05807}, 2016.

\bibitem[Kuleshov et~al.(2018)Kuleshov, Fenner, and
  Ermon]{kuleshov2018accurate}
Kuleshov, V., Fenner, N., and Ermon, S.
\newblock Accurate uncertainties for deep learning using calibrated regression.
\newblock \emph{arXiv preprint arXiv:1807.00263}, 2018.

\bibitem[Kusner et~al.(2017)Kusner, Loftus, Russell, and
  Silva]{kusner2017counterfactual}
Kusner, M.~J., Loftus, J., Russell, C., and Silva, R.
\newblock Counterfactual fairness.
\newblock In \emph{Advances in Neural Information Processing Systems}, pp.\
  4066--4076, 2017.

\bibitem[Lakshminarayanan et~al.(2017)Lakshminarayanan, Pritzel, and
  Blundell]{lakshminarayanan2017simple}
Lakshminarayanan, B., Pritzel, A., and Blundell, C.
\newblock Simple and scalable predictive uncertainty estimation using deep
  ensembles.
\newblock In \emph{Advances in neural information processing systems}, pp.\
  6402--6413, 2017.

\bibitem[Levi et~al.(2019)Levi, Gispan, Giladi, and Fetaya]{levi2019evaluating}
Levi, D., Gispan, L., Giladi, N., and Fetaya, E.
\newblock Evaluating and calibrating uncertainty prediction in regression
  tasks.
\newblock \emph{arXiv preprint arXiv:1905.11659}, 2019.

\bibitem[Liu et~al.(2018)Liu, Simchowitz, and Hardt]{liu2018implicit}
Liu, L.~T., Simchowitz, M., and Hardt, M.
\newblock The implicit fairness criterion of unconstrained learning.
\newblock \emph{arXiv preprint arXiv:1808.10013}, 2018.

\bibitem[Louizos et~al.(2015)Louizos, Swersky, Li, Welling, and
  Zemel]{louizos2015variational}
Louizos, C., Swersky, K., Li, Y., Welling, M., and Zemel, R.
\newblock The variational fair autoencoder.
\newblock \emph{arXiv preprint arXiv:1511.00830}, 2015.

\bibitem[Malik et~al.(2019)Malik, Kuleshov, Song, Nemer, Seymour, and
  Ermon]{malik2019calibrated}
Malik, A., Kuleshov, V., Song, J., Nemer, D., Seymour, H., and Ermon, S.
\newblock Calibrated model-based deep reinforcement learning.
\newblock \emph{arXiv preprint arXiv:1906.08312}, 2019.

\bibitem[Murphy(1973)]{murphy1973new}
Murphy, A.~H.
\newblock A new vector partition of the probability score.
\newblock \emph{Journal of applied Meteorology}, 12\penalty0 (4):\penalty0
  595--600, 1973.

\bibitem[Myers \& Myers(1990)Myers and Myers]{myers1990classical}
Myers, R.~H. and Myers, R.~H.
\newblock \emph{Classical and modern regression with applications}, volume~2.
\newblock Duxbury press Belmont, CA, 1990.

\bibitem[Niculescu-Mizil \& Caruana(2005)Niculescu-Mizil and
  Caruana]{niculescu2005predicting}
Niculescu-Mizil, A. and Caruana, R.
\newblock Predicting good probabilities with supervised learning.
\newblock In \emph{Proceedings of the 22nd international conference on Machine
  learning}, pp.\  625--632. ACM, 2005.

\bibitem[Perdomo et~al.(2020)Perdomo, Zrnic, Mendler-D{\"u}nner, and
  Hardt]{perdomo2020performative}
Perdomo, J.~C., Zrnic, T., Mendler-D{\"u}nner, C., and Hardt, M.
\newblock Performative prediction.
\newblock \emph{arXiv preprint arXiv:2002.06673}, 2020.

\bibitem[Pfohl et~al.(2019)Pfohl, Marafino, Coulet, Rodriguez, Palaniappan, and
  Shah]{pfohl2019creating}
Pfohl, S., Marafino, B., Coulet, A., Rodriguez, F., Palaniappan, L., and Shah,
  N.~H.
\newblock Creating fair models of atherosclerotic cardiovascular disease risk.
\newblock In \emph{Proceedings of the 2019 AAAI/ACM Conference on AI, Ethics,
  and Society}, pp.\  271--278, 2019.

\bibitem[Platt et~al.(1999)]{platt1999probabilistic}
Platt, J. et~al.
\newblock Probabilistic outputs for support vector machines and comparisons to
  regularized likelihood methods.
\newblock \emph{Advances in large margin classifiers}, 10\penalty0
  (3):\penalty0 61--74, 1999.

\bibitem[Pleiss et~al.(2017)Pleiss, Raghavan, Wu, Kleinberg, and
  Weinberger]{pleiss2017fairness}
Pleiss, G., Raghavan, M., Wu, F., Kleinberg, J., and Weinberger, K.~Q.
\newblock On fairness and calibration.
\newblock In \emph{Advances in Neural Information Processing Systems}, pp.\
  5680--5689, 2017.

\bibitem[Shalev-Shwartz et~al.(2012)]{shalev2012online}
Shalev-Shwartz, S. et~al.
\newblock Online learning and online convex optimization.
\newblock \emph{Foundations and Trends{\textregistered} in Machine Learning},
  4\penalty0 (2):\penalty0 107--194, 2012.

\bibitem[Song et~al.(2018)Song, Kalluri, Grover, Zhao, and
  Ermon]{song2018learning}
Song, J., Kalluri, P., Grover, A., Zhao, S., and Ermon, S.
\newblock Learning controllable fair representations.
\newblock \emph{arXiv preprint arXiv:1812.04218}, 2018.

\bibitem[Trejo et~al.(2015)Trejo, Clempner, and Poznyak]{trejo2015stackelberg}
Trejo, K.~K., Clempner, J.~B., and Poznyak, A.~S.
\newblock A stackelberg security game with random strategies based on the
  extraproximal theoretic approach.
\newblock \emph{Engineering Applications of Artificial Intelligence},
  37:\penalty0 145--153, 2015.

\bibitem[Tsai et~al.(2010)Tsai, Yin, Kwak, Kempe, Kiekintveld, and
  Tambe]{tsai2010urban}
Tsai, J., Yin, Z., Kwak, J.-y., Kempe, D., Kiekintveld, C., and Tambe, M.
\newblock Urban security: Game-theoretic resource allocation in networked
  domains.
\newblock In \emph{Twenty-Fourth AAAI Conference on Artificial Intelligence},
  2010.

\bibitem[Vovk(2012)]{vovk2012conditional}
Vovk, V.
\newblock Conditional validity of inductive conformal predictors.
\newblock In \emph{Asian conference on machine learning}, pp.\  475--490, 2012.

\bibitem[Vovk et~al.(2005)Vovk, Gammerman, and Shafer]{vovk2005algorithmic}
Vovk, V., Gammerman, A., and Shafer, G.
\newblock \emph{Algorithmic learning in a random world}.
\newblock Springer Science \& Business Media, 2005.

\bibitem[Zadrozny \& Elkan(2001)Zadrozny and Elkan]{zadrozny2001obtaining}
Zadrozny, B. and Elkan, C.
\newblock Obtaining calibrated probability estimates from decision trees and
  naive bayesian classifiers.
\newblock In \emph{Icml}, volume~1, pp.\  609--616. Citeseer, 2001.

\bibitem[Zadrozny \& Elkan(2002)Zadrozny and Elkan]{zadrozny2002transforming}
Zadrozny, B. and Elkan, C.
\newblock Transforming classifier scores into accurate multiclass probability
  estimates.
\newblock In \emph{Proceedings of the eighth ACM SIGKDD international
  conference on Knowledge discovery and data mining}, pp.\  694--699. ACM,
  2002.

\bibitem[Zemel et~al.(2013)Zemel, Wu, Swersky, Pitassi, and
  Dwork]{zemel2013learning}
Zemel, R., Wu, Y., Swersky, K., Pitassi, T., and Dwork, C.
\newblock Learning fair representations.
\newblock In \emph{International Conference on Machine Learning}, pp.\
  325--333, 2013.

\end{thebibliography}

\newpage
\newcommand{\ece}{{\mathrm{ECE}}}

\onecolumn
\appendix

\section{Experiments Details and Additional Results}

\subsection{Additional Theoretic Results}
\subsubsection{Relationship between PAIC and ECE}
\label{sec:appendix_ece}
Given a forecaster $\fs$ we can define expected calibration error (ECE) as 
\begin{align*}
    \ece(\fs) = \int_{c=0}^1 \left\lvert \Pr[\fs[\xs](\ys) \leq c] - c \right\rvert dc
\end{align*}

\begin{restatable}{prop}{eqece}
\label{prop:eq_ece} 
$\ece(\fs) = d_{W1}(\fcdfs, \ucdf)$.
\end{restatable}
Intuitively, both $d_{W1}(\fcdfs, \ucdf)$ try to integrate the difference between the curve $c \mapsto \Pr[\fs[\xs](\ys) \leq c]$ and the curve $c \mapsto c$. The difference is that they integrate the difference in different ways (similar to the difference between Riemann and Lebesgue integral). 

\subsubsection{Trivial Construction of mPAIC forecaster}
\label{sec:appendix_trivial_mpaic}

We construct a trivial forecaster that is always mPAIC. Let $\Phi$ be the standard Gaussian CDF, In particular for some $c > 0$, choose 
\begin{align*}
    \fdr[x, r](y) = \Phi(y/c - \Phi^{-1}(r))
\end{align*}
Then when $c \to \infty$, we have $ \fdr[x, r](y) = \Phi(\Phi^{-1}(r)) = r$. In other words, for any $\epsilon, \delta$, $\fdr$ is $(\epsilon,\delta)$-mPAIC for sufficiently large $c$. However, this forecaster is certainly not useful in practice because it outputs a distribution with variance $\to \infty$.

\subsection{Fairness Experiment Details}

We use the UCI crime and communities dataset~\citep{dua2019uci} and we predict the crime rate based on features about the neighborhood (such as racial composition). The prediction model is a fully connected deep network, where the additional input $r$ is concatenated into each hidden layer (except the last one). %We train on the objective
Other than this difference, all other setups are standard --- with dropout and early stopping on validation data to prevent over-fitting. For details please refer to the code included with this paper. 

During evaluation of calibration error for interpretable groups, we only consider groups with at least 150 samples to avoid excessive estimation error. 
%\s{explain what you are predicting}

\subsection{Additional Plots and Comparisons for Fairness Experiments}

In Figure~\ref{fig:ece_crime2} we plot the same experimental results in Figure~\ref{fig:ece_crime}, where the only difference is we apply post-training recalibration~\citep{kuleshov2018accurate}. There is no qualitative difference between Figure~\ref{fig:ece_crime} and Figure~\ref{Fig:ece_crime2} because (average) calibration does not improve calibration for the worst group. 

\begin{figure*}
\begin{center}
\includegraphics[width=0.94\linewidth]{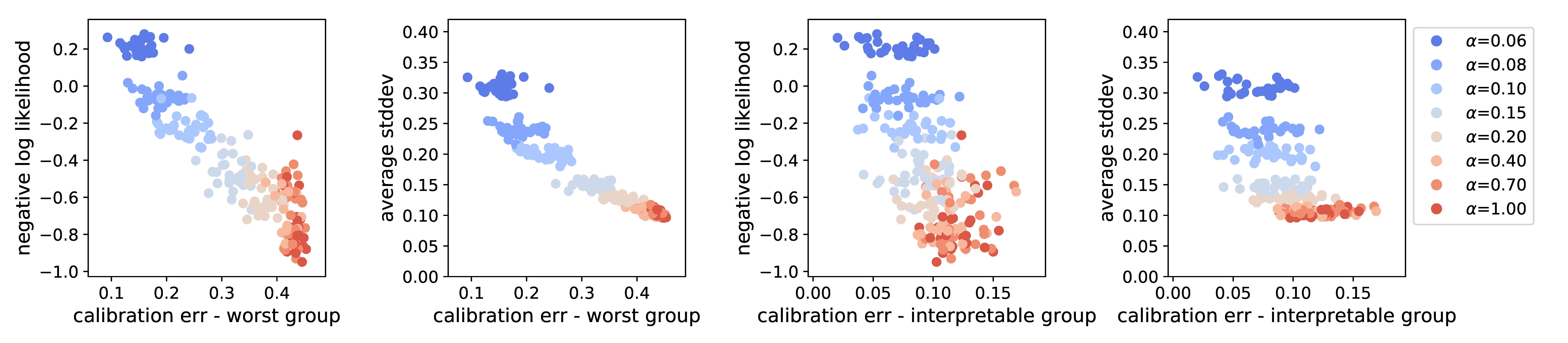}
\vspace{-3mm}
\caption{The same plot as Figure~\ref{fig:ece_crime}, but with recalibration by isotonic regression. The results are qualitatively the same with and without recalibration.
%\tnote{in the xlable, can we use interpretable group instead of just interpretable?}
}
\label{fig:ece_crime2}
\end{center}
\vspace{-3mm}
\end{figure*}

\begin{table}
    \centering
    \begin{tabular}{c|c|c}
         & No Recalibration & With Recalibration \\
         \hline
        $\alpha=0.1$ & population density (+) & pct immigrant 8yr (-)  \\ 
        (PAIC) & pct drug officer (-) & vac house boarded (-)\\ \hline
        $\alpha=1.0$ & pct black (+) & pct immigrant (+) \\
        (NLL) & pct < 3 bedroom (+) & pct dense house (+) 
    \end{tabular}
    \caption{Least calibrated group for each setup. A + sign indicates this feature is above the median and a - sign indicates the feature is below the median. These are indeed groups where fairness can be a consideration (e.g. immigrants, race or economic condition).}
    \label{table:groups}
\end{table}

\subsection{Experiment Details for Credit Approval}
\label{sec:appendix_credit_details}
\label{sec:appendix_credit_guarantee}

\paragraph{Dataset} We will use the "Give Me Some Credit" dataset on Kaggle. Because it is a binary classification dataset (credit delinquency vs. no delinquency), we first train a classifier to predict the Bernoulli probability, and use the probability (plus a small Gaussian noise) as the label. We synthesize a training set and a validation set, where the validation set is very large to simulate a stream of non-repeating customers. We train the bank's forecaster $\fs$ on the training set, and apply it to interacting with customers sampled from the validation set.

\paragraph{Customer Model} The customer utility we use is

\begin{center}
\begin{tabular}{c|cc}
   & $y \geq y_0$ & $y < y_0$ \\ \hline
    'yes' & 0.2 & 1.0 \\
    'no' & -0.5 & -0.5
\end{tabular}
\end{center}

We assume the customers knows their own credit worthiness $y$. Based on previous customers $x, y$ and the actual utility from playing the game, we learn a function $\psi(x, y) \to \mathbb{R}$ by gradient descent to predict the customer's utility. The prediction function $\psi$ is also a fully connected deep neural network. Each new customer $(x_{\text{new}}, y_{\text{new}}) \sim \cdfjoint$ will only apply if $\psi(x_{\text{new}}, y_{\text{new}}) \geq 0$.

\paragraph{Decision Rule}
The ``Bayesian'' decision rule in Eq.(\ref{eq:bayesian_decision}) can be written as
\begin{align*}
     \phi_\fs(x) = \left\lbrace \begin{array}{cc} \text{'yes'} & \fs[x](y_0) \leq 1/4 \\ \text{'no'} & \text{otherwise} \end{array} \right. \numberthis\label{eq:bank_strategy}
\end{align*}

\textbf{Recalibration} Since post training recalibration~\citep{kuleshov2018accurate,malik2019calibrated} is usually beneficial, we will report both results with and without recalibration by isotonic regression. The results with recalibration is in Figure~\ref{fig:credit_simulation} and the results without recalibration is in Figure~\ref{fig:credit_simulation2}.

% \subsection{Simulation Setup}

% \s{i don't think these two sections provide enough info about what is going on. either you cut entirely this 6.1.2 and refer everything to appendix, or might need to say a bit more so the setup is understandable}

% We will run some simulations to verify that individual calibrated forecasters are less exploitable and achieve higher utility in practice --- when there is finite training data and approximate calibration.
% %where calibration is only approximately achieved, individual calibration reduces exploitability and leads to higher overall utility for the bank. 
% Details about the dataset ("Give Me Some Credit" on Kaggle) are in Appendix~\ref{sec:appendix_credit}.

% \textbf{Customer Model} The utility for the customer in our simulation is in Appendix~\ref{sec:appendix_credit}. 

\subsection{Additional Plots for Credit Approval}

In Figure~\ref{fig:credit_simulation2} we plot the results without post training recalibration. They are qualitatively similar to Figure~\ref{fig:credit_simulation}. Post training recalibration has little effect on calibration of the worst sub-group, and therefore do not improve performance in this experiment. 

\begin{figure}
    \centering
    %\adjincludegraphics[height=5cm,trim={0 0 {.5\width} 0},clip]{example-image-a}
    \includegraphics[width=0.5\linewidth]{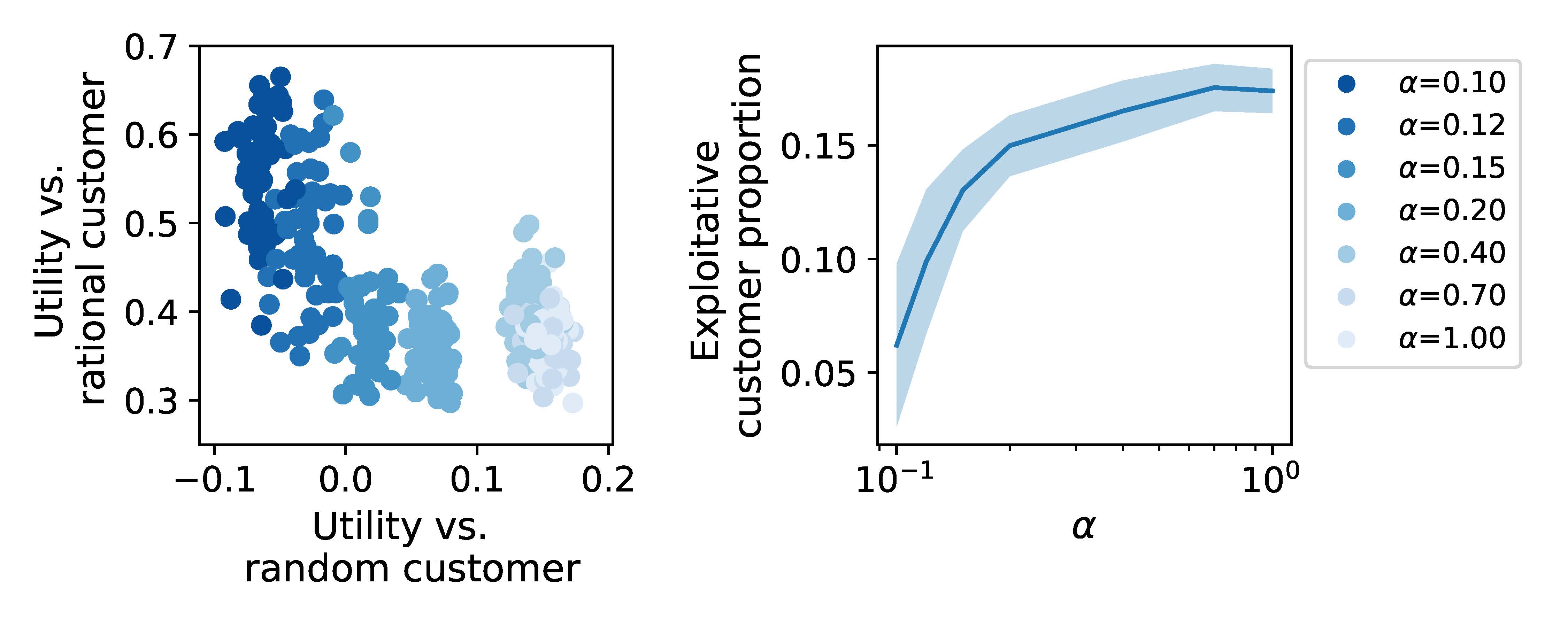}
    \vspace{-5mm}
    \caption{The experiment in Figure~\ref{fig:credit_simulation} without post training recalibration. 
    %\textbf{Right 2}: The same experiment, but with post training recalibration on the bank's forecaster. \s{what's the takeway of the right vs left plots?} 
    }
    \label{fig:credit_simulation2}
\end{figure}

% \begin{theorem}
% \label{thm:paicmono_imply_paic}
% If $\fdr$ is $(\epsilon, \delta)$-mPAIC, then for any $\epsilon' > 0$ it is $(\epsilon', \delta+(1-\delta)\epsilon/\epsilon')$-PAIC with respect to the 1-Wasserstein distance. 
% \end{theorem}

% \subsection{General Distance Function}

\section{Proofs}

\subsection{Proofs for Section 3}
\label{sec:appendix_impossibility}

\impossibledeterministic*

\begin{proof}[Proof of Proposition~\ref{prop:impossible_deterministic}]
%We assume implicitly that $\cdf$ satisfies that $x,x'\sim \cdf$ are distinct a.s. This is true in general when $\Xc$ is a continuous space, and $\cdf_X$ assigns zero measure to every $x \in \Xc$
Given a distribution $\cdf$ and forecaster $\fd$ such that 
$\Pr_{\Dc \sim \cdf}[T(\Dc, \fd) = \yes] = \kappa > 0$, 
%choose any $\Dc$ such that $T_{\epsilon,\delta}(\Dc, \fd) = \yes$, and 
we will construct an alternative distribution $\cdf'$ by choosing some function $g: \Xc \to \Yc$ (defined later), and define a new distribution $\xs', \ys' \sim \cdf'_g$ by: $\xs' \sim \cdf_\xs$ and $\ys' \mid x$ is the delta distribution on $g(x)$. Then by Definition~\ref{def:individual_calibration}, $\forall x \in \Xc$ 
\begin{align*} \dw(\cdf_{\fd[x](\ys')}, \ucdf) \geq 1/4 \numberthis\label{eq:trivial_solution} \end{align*} 
In words, the above expression is because for any distribution $\fd[x]$ outputs, we can never rule out a possible ground truth distribution ($\cdf'_g$) that is deterministic. Under a deterministic distribution $\cdf_{\ys' \mid x}$, it must be that Eq.(\ref{eq:trivial_solution}) holds. (An alternative construction can strengthen the theorem by choosing $\ys' \mid x$ to be a distribution with sufficiently small non-zero variance. It can become clear that Eq.(\ref{eq:trivial_solution}) is not an artifact of our requirement that $\fd$ must output a continuous CDF, but rather the the variance of the ground truth distribution cannot be known). 

%\s{should this be $\geq$? i think you should emphasize this holds for every h, and any g}
%\s{also show here where you use the fact x are different}
%\s{and tie back explicitly to definition 1}

What remains to show is that there must exist a $g$ such that 
\[ \Pr_{\Dc \sim \cdf'_g}[T(\Dc, \fd) = \yes] \geq \kappa \]
We will do this with the probabilistic method. For convenience we will represent the value `yes' by 1 and the value `no' by 0. We can use the notation
\[ \Pr_{\Dc \sim \cdf} [T(\Dc, \fd) = \yes] := \Eb_{\Dc \sim \cdf}[T(\Dc, \fd)] \]
%\s{using F and D notation could be confusing, consider D'}\sj{Have to do this later because we cannot change the main paper}
Because $\cdf_\xs$ assigns zero measure to individual points, for any finite set of $x_1, \cdots, x_n \simiid \cdf_\xs$, all the $x_i$ are distinct (i.e. $x_i \neq x_j, \forall i \neq j$) almost surely. 
Suppose $\Gb$ is a random function on $\lbrace g: \Xc \to \Yc\rbrace$ defined by $\Gb(x) \sim \cdf_{\ys \mid x}$, then random variables $\Dc = \lbrace (x_1, y_1), \cdots, (x_n, y_n) \rbrace$ defined by the following two sampling procedures are identically distributed (i.e. in the sense that they belong to any measurable subset of $(\Xc \times \Yc)^n$ with the same probability) 
\begin{align*}
    &x_1, \cdots, x_n \simiid \cdf_\xs, \qquad y_i \sim \cdf_{\ys \mid x_i} \\
    &x_1, \cdots, x_n \simiid \cdf_\xs, \qquad g \sim \Gb, \qquad y_i = g(x_i) 
\end{align*}
%\s{i think this part is sloppy/incorrect. these 2 distributions are not the same, need to carefully handle the different x}
In words, we could either 1. directly sample a dataset $\Dc$ from $\cdf$, or 2. we could first sample a value $g(x) \sim \cdf_{\ys \mid x}$ for each $x$, then sample $x_1, \cdots, x_n \sim \cdf_\xs$ and directly evaluate $y_1 = g(x_1), \cdots, y_n = g(x_n)$. %Because almost surely $$ %Because $\cdf_\xs$ assigns zero measure to any individual point $x \in \Xc$, for any measurable set $\mathfrak{D} \subset (\Xc \times \Yc)^n$, $\Pr[\Dc \in \mathfrak{D}]$ is the same for both random variables (this would not be true if $\cdf_\xs$ assigns non-zero measure to any point). 

Therefore any bounded random variable must have identical expectation under the probability law defined by the two sampling procedures 
\[ \kappa = \Eb_{\Dc \sim \cdf}[T(\Dc, \fd)] = \Eb_{g \sim \Gb}\Eb_{\Dc \sim \cdf'_g}[T(\Dc, \fd)] \]
But this must imply there exists $g$ such that  % \s{expand this, might not be obvious}
\[ \Eb_{\Dc \sim \cdf'_g}[T(\Dc, \fd)] \geq \kappa \]
because a random variable must be able to take a value that is at least its expectation

For any other divergence we can get similar results by replacing Eq.(\ref{eq:trivial_solution}). For example, for total variation distance we have
\begin{align*} d_{\mathrm{TV}}(\cdf_{\fd[x](\ys')}, \ucdf) = 1 \numberthis\label{eq:trivial_solution} \end{align*} 
\end{proof}

\begin{prop}%[Informal]\tnote{say "Informal version of Theorem ??"}
\label{prop:impossible_deterministic_advgroup}
For any distribution $\cdf$ on $\Xc \times \Yc$ such that $\cdf_\xs$ assigns zero measure to individual points $\lbrace x \in \Xc \rbrace$ sample $\Dc = \lbrace (x_1, y_1), \cdots, (x_n, y_n) \rbrace \simiid \cdf$. % and assume $x_i \neq x_j, \forall i \neq j$ almost surely.
 For any deterministic forecaster $h$ and any function
$ T(\Dc, \fd) \to \lbrace \yes, \no \rbrace $ %\tnote{$T$ to $T_{\epsilon,\delta}$? check other occurences.}\sj{Actually we don't need $\epsilon,\delta$, this is true for any function $T$}
such that 
\[ \Pr_{\Dc \sim \cdf}[T(\Dc, \fd) = \yes] = \kappa > 0\,, \]
then there exists a distribution $\cdf'$, $\fd$ is not $\left( \epsilon, \delta \right)$-adversarial group calibrated with respect to $\cdf'$ for any $\epsilon < 1/8$ and $\delta < 1/2$, and 
\[ \Pr_{\Dc \sim \cdf'}[T(\Dc, \fd) = \yes] \geq \kappa \]
\end{prop}

\begin{proof}[Proof of Proposition~\ref{prop:impossible_deterministic_advgroup}]
The proof is almost identical to the proof of Proposition~\ref{prop:impossible_deterministic}. We construct a $g: \Xc \to \Yc$ such that 
\[ \Pr_{\Dc \sim \cdf'}[T(\Dc, \fd) = \yes] \geq \kappa \]
We can pick the subgroup $\Sc_1, \Sc_2 \subset \Xc$ defined by 
\[ \Sc_1 = \lbrace x, \fd[x](g(x)) \geq 1/2 \rbrace, \Sc_2 = \lbrace x, \fd[x](g(x)) < 1/2 \rbrace \]
Because $\Sc_1 \cup \Sc_2 = \Xc$, so at least one of $\Sc_1, \Sc_2$ must have probability measure at least $1/2$ under $\cdf_X$. Without loss of generality assume it's $\Sc_1$. Then for $\tilde{X} = X \mid \Sc_1$ we have $\fd[\tilde{X}](g(x)) \geq 1/2$ almost surely, $\cdf_{\fd[\tilde{X}](g(x))}(r) = 0, \forall r \in [0, 1/2]$, which implies
\[ d_{W1}(\cdf_{\fd[\tilde{X}](g(x))}, \ucdf) \geq 1/8 \]
Therefore $\fd$ cannot be $(1/8, 1/2)$-adversarial group calibrated. 
\end{proof}

% \begin{proof}[Proof of Proposition~\ref{prop:impossibility_group}]
% As the attacker is arbitrary, he can attack based on the true distribution $p^*$ and the calibrated model $\hat{f}$. He computes $p^*(\hat{f}(x)(t^*(x)) \not\in G)$, if this is greater than $1/2$ he plays $G'=G$, otherwise he plays $G'=G^c$. During the game play, he chooses $A(x)=1$ whenever $\hat{f}(x)(t^*(x)) \not\in G'$. Note that this must happen with probability at least $1/2$. He will always win, so $\Eb[C|A(x)=1] = \frac{1}{2|G'|} = 1$
% \end{proof}

% \begin{proof}[Proof of Lemma~\ref{lemma:suff_condition}]
% Suppose for some $x$,  $\int_r \lvert f(x, r)(y) - r  dr \rvert \geq \epsilon$. We denote such a set of $x$ by $\Xc_G \subset \Xc$, then for any $x \in \Xc_G$, by Markov inequality,
% \[ \Pr_r\left[ \lvert f(x, r)(y) - r \rvert \geq \epsilon'\right] \leq \epsilon/\epsilon' \]
% Then we have
% \begin{align*}
%     &\Pr_{x, y, r} \left[ |f(x, r)(y) - r| \geq \epsilon' \right] \\
%     &=  \Pr_{x, y, r} \left[ |f(x, r)(y) - r| \geq \epsilon' \right | x \in \Xc_G] \Pr[\Xc_G] \\
%     &\quad + \Pr_{x, y, r} \left[ |f(x, r)(y) - r| \geq \epsilon' \right | x \not\in \Xc_G] (1 - \Pr[\Xc_G]) \\
%     &\leq (1 - \delta)\epsilon/\epsilon' + \delta
% \end{align*}
% \end{proof}

\subsection{Proofs for Section 4.2}
\label{sec:appendix_mono_paic}

\paicmono*
\begin{proof}[Proof of Theorem~\ref{thm:paicmono_imply_paic}]

% Suppose for some $x$,  $\int_r \lvert f(x, r)(y) - r  dr \rvert \geq \epsilon$. We denote such a set of $x$ by $\Xc_G \subset \Xc$, then for any $x \in \Xc_G$, by Markov inequality,
% \[ \Pr_r\left[ \lvert f(x, r)(y) - r \rvert \geq \epsilon'\right] \leq \epsilon/\epsilon' \]
% Then we have
% \begin{align*}
%     &\Pr_{x, y, r} \left[ |f(x, r)(y) - r| \geq \epsilon' \right] \\
%     &=  \Pr_{x, y, r} \left[ |f(x, r)(y) - r| \geq \epsilon' \right | x \in \Xc_G] \Pr[\Xc_G] + \Pr_{x, y, r} \left[ |f(x, r)(y) - r| \geq \epsilon' \right | x \not\in \Xc_G] (1 - \Pr[\Xc_G]) \\
%     &\leq (1 - \delta)\epsilon/\epsilon' + \delta
% \end{align*}
Recall the convention that $\ys$ is the random variable that always has the conditional distribution $\cdf_{\ys \mid x}$, and $\rs$ is uniformly distributed in $[0,1]$. Denote 
\[ \err(x, y) := d_{W_1}\left(\cdf_{\fdr[x, \rs](y)}, \ucdf\right), \qquad \err(x) := d_{W_1}\left(\cdf_{\fdr[x, \rs](\ys)}, \ucdf\right) \]
Suppose $\fdr[x, \cdot](y)$ is a monotonically non-decreasing function for all $x, y$, then 
\[ \err(x, y) = d_{W_1}\left(\cdf_{\fdr[x, \rs](y)}, \ucdf\right) = \intr \lvert \fdr(x, r)(y) - r \rvert dr = \Eb\left[ \lvert \fdr(x, \rs)(y) - \rs \rvert \right] \]
% where the second equality is because $\cdf_{\fdr[x, \rs](y)}(s) = \Pr[\fdr[x, \rs](y) \leq s] = \fdr[x, r](y)$, by applying the formula for computing one-dimensional Wasserstein distance. 
So in general for arbitrary $\fdr$ we have 
\begin{align*}
    \err(x, y) = d_{W_1}\left(\cdf_{\fdr[x, \rs](y)}, \ucdf\right) \leq \intr \lvert \fdr(x, r)(y) - r \rvert dr  = \Eb\left[ \lvert \fdr(x, \rs)(y) - \rs \rvert \right]  \numberthis\label{eq:mpaic_to_paic_0}
\end{align*} 
In addition by Jensen's inequality we have $\err(x) \leq \Eb[\err(x, \ys)]$ so
\begin{align*}
    \err(x) \leq \Eb\left[ \lvert \fdr(x, \rs)(\ys) - \rs \rvert \right] \numberthis\label{eq:mpaic_to_paic_1}
\end{align*}
Suppose $\fdr$ is not $(\epsilon',\delta')$-PAIC, by definition we have
\[ \Pr[\err(\xs) \geq \epsilon'] > \delta' \] 
% we can conclude 
% \begin{align*}
%     \Pr\left[  \Eb\left[ \lvert \fdr(x, \rs)(\ys) - \rs \rvert \right]  \geq \epsilon' \right] > \delta'
% \end{align*}
% % because $\err(x) \leq \Eb[\err(x, \ys)]$, and by Eq.(\ref{eq:mpaic_to_paic_0}) we can conclude
% \begin{align*}
%  \Pr[\Eb[\intr \lvert \fdr(\xs, r)(\ys) - r \rvert dr \geq \epsilon'] > \delta' 
% \end{align*}
% we can conclude
% \begin{align*}
%   \Pr[\err(\xs, \ys) \geq \epsilon'] > \delta' 
% \end{align*}
% By Eq.(\ref{eq:mpaic_to_paic_0}) we can further conclude
% \begin{align*}
%     \Pr\left[\intr \lvert \fdr(\xs, r)(\ys) - r \rvert dr \geq \epsilon'\right] > \delta'
% \end{align*}
Define the notation 
\begin{align*}
    \Sc_b := \lbrace x \in \Xc,  \Eb\left[ \lvert \fdr(x, \rs)(\ys) - \rs \rvert \right]  \geq \epsilon' \rbrace 
\end{align*} 
by Eq.(\ref{eq:mpaic_to_paic_1}) we know that whenever $\err(x) \geq \epsilon'$ we have $x \in \Sc_b$, so we can conclude 
\begin{align*}
    \Pr[\xs \in \Sc_b] > \delta' \numberthis\label{eq:mpaic_to_paic_3}
\end{align*}
%and as the complement of $\Sc_b$ in $\Xc \times \Yc$. 
Whenever $x \in \Sc_b$, for any $\epsilon < \epsilon'$, we have
\begin{align*}
    \epsilon' &\leq \Eb[|\fdr(x, \rs)(\ys) - \rs|] \\
    &\leq \epsilon \Pr[|\fdr(x, \rs)(\ys) - \rs| < \epsilon] + \Pr[|\fdr(x, \rs)(\ys) - \rs| \geq \epsilon] \\
    &= \epsilon(1-\Pr[|\fdr(x, \rs)(\ys) - \rs| \geq \epsilon]) + \Pr[|\fdr(x, \rs)(\ys) - \rs| \geq \epsilon]
\end{align*}
where the second inequality is because $|\fdr(x, \rs)(\ys) - \rs| $ is bounded in $[0, 1]$. By simple algebra we get 
\begin{align*}
    \Pr[|\fdr(x, \rs)(\ys) - \rs| \geq \epsilon] \geq \frac{\epsilon'-\epsilon}{1 - \epsilon}     \numberthis\label{eq:mpaic_to_paic_2}
\end{align*}
%In addition we must have $\Pr[\xs, \ys \in \Sc_b] > \delta'$. If this is not true, then $\fdr$ is $(\epsilon', \delta')$-mPAIC, which implies it must also be $(\epsilon', \delta')$-PAIC (because mPAIC is a sufficient 
We can combine Eq.(\ref{eq:mpaic_to_paic_3}) and Eq.(\ref{eq:mpaic_to_paic_2}) to get
\begin{align*}
    &\Pr[ \lvert \fdr(\xs, \rs)(\ys) - \rs \rvert \geq \epsilon']\\
    &= \Pr[ \lvert \fdr(\xs, \rs)(\ys) - \rs \rvert \geq \epsilon \mid \xs \in \Sc_b]\Pr[\xs \in \Sc_b] + \Pr[ \lvert \fdr(\xs, \rs)(\ys) - \rs \rvert \geq \epsilon \mid \xs \not\in \Sc_b]\Pr[\xs \not\in \Sc_b] \\
    &> \frac{\epsilon' - \epsilon}{1 - \epsilon} \delta'
\end{align*}
Therefore, $\fdr$ is not $(\epsilon, \frac{\epsilon' - \epsilon}{1 - \epsilon} \delta')$-mPAIC. To summarize, we have concluded that whenever $\fdr$ is not $(\epsilon', \delta')$-PAIC, for any $\epsilon < \epsilon'$, it is not  $(\epsilon, \frac{\epsilon' - \epsilon}{1 - \epsilon} \delta')$-mPAIC. This is equivalent to the statement: suppose $\fdr$ is $(\epsilon, \delta)$-mPAIC, then $\fdr$ is $(\epsilon', \delta \frac{1-\epsilon}{\epsilon'-\epsilon})$-PAIC. 
\end{proof}

% We can also extend this result to any Wasserstein-$s$ distance by the following definition 
% \begin{definition}
% \label{def:w-pai}
% A forecaster $\fdr$ is $(\epsilon,\delta)$-mPAIC with respect to $\Wc_s$ if
% \[ \Pr \left[ \lvert \fdr[\xs, \rs](\ys) - r \rvert^s dr \geq \epsilon\right] \leq \delta \]
% \end{definition}

\concentration*
\begin{proof}[Proof of Proposition~\ref{prop:concentration}]
Consider the sequence of Bernoulli random variables $b_i = \mathbb{I}(\left\lvert \fdr(x_i, r_i)(y_i) - r_i \right\rvert \geq \epsilon)$. Suppose $\Eb[b_i] = \delta$, then by Hoeffding inequality 
\[ Pr\left[\frac{1}{n}\sum_{i} b_i \geq \delta + \epsilon \right] \leq e^{-2\epsilon^2 n} \]
Plugging in $e^{-2\epsilon^2T}$ as $\gamma$ we have $ \epsilon = \sqrt{\frac{-\log \gamma}{2n}} $
\end{proof}

\subsection{Proofs for Section 5}
\label{sec:appendix_fairness}

\advgroup* 

\begin{proof}[Proof of Theorem~\ref{thm:pai_vs_calibration}]
Given a forecaster $\fs$ if for some $x, y$ we have $d_{\Wc_p}(\cdf_{\fs[x](y)}, \ucdf) < \epsilon$ then by definition of the Wasserstein distance we have
\[ \int_{r=0}^1 \lvert \fcdfd(r) - r \rvert^p \leq \epsilon^p \]
If $\fs$ is $(\epsilon,\delta)$-mPAIC with respect to $\Wc_p$.  Consider a partition of $\Xc$ into two sets: $\Xc_g$ where $\forall x \in \Xc_g$ we have
\[ \int_r \left\lvert \cdf_{\fs[x](y)}(r) - r\right\rvert^p dr \leq \epsilon^p \]
and $\Xc_b$ where the above property fails. We know $\Pr[\Xc_b] \leq \delta$. In general for any $x \in \Xc$, because $\fcdfd$ is a monotonically increasing function of $c$ bounded in $[0, 1]$, we have
\begin{align*}
     \int_r \left\lvert \fcdfd(r) - r \right\rvert^p dc \leq \int_r r^p dr = \frac{1}{p+1}
\end{align*}

Another useful identity we will use is 
\begin{align*} 
\fcdfs(r) = \Pr[\fs[\xs][\ys] \leq r] = \Eb_{x \sim \cdf_\xs}[\Eb[\mathbb{I}(\fs[x][\ys] \leq r)]] = \Eb_{x \sim \cdf_\xs}[\fcdf(r)] \numberthis\label{eq:1}
\end{align*}
Combining the above results we have for any $\tilde{X} = X \mid \mathcal{S}$
\begin{align*}
    d_{\Wc_p}(\cdf_{\fs[\tilde{X}](Y)}, \ucdf) &= \left( \intr \lvert \cdf_{\fs[\tilde{X}](Y)}(r) - r \rvert^p dr \right)^{1/p} \\
    &= \left( \intr \left\lvert \Eb_{x \sim \tilde{X}}[\cdf_{\fs[x](\Yb)}(r) - r] \right\rvert^p  dr \right)^{1/p} & (\text{Eq.\ref{eq:1}}) \\
    &\leq \Eb_{x \sim \tilde{X}}\left[ \left( \intr |\cdf_{\fs[x](\Yb)}(r) - r \rvert^p dr \right)^{1/p} \right]  & \textrm{(Jensen)} \\
    &= \Eb_{x \sim \tilde{X}}\left[\left( \intr |\cdf_{\fs[x](\Yb)}(r) - r \rvert^p dr \right)^{1/p}  \mid x \in \Xc_g\right] \Pr[\tilde{X} \in \Xc_g] + \\
    & \qquad \Eb_{x \sim \tilde{X}}\left[ \left( \intr |\cdf_{\fs[x](\Yb)}(r) - r \rvert^p dr \right)^{1/p}  \mid x \in \Xc_b\right] \Pr[\tilde{X} \in \Xc_b] & \textrm{(Conditional Expectation)} \\
    &\leq \epsilon \Pr[\tilde{X} \in \Xc_g] + (p+1)^{-1/p} \Pr[\tilde{X} \in \Xc_b] \\
    &\leq \epsilon \frac{\delta'-\delta}{\delta'} + (p+1)^{-1/p} \frac{\delta}{\delta'} & \textrm{(} \epsilon \leq (p+1)^{-1/p} \textrm{)}
\end{align*}
% \begin{align*}
%     & \left\lVert  Pr[\psi_x(r) < c|g(x)=1] - c \right\rVert_s =  \left\lVert \Eb_{x}\left[\zeta(x, c)- c  \vert x \in \Gc \right] \right\rVert_s  \\
%     &\leq \Eb_{x}\left[ \left\lVert \zeta(x, c)- c  \right\rVert_s \right \vert x \in \Gc]  \\
%     &= \Eb_{x}\left[ \left\lVert \zeta(x, c)- c  \right\rVert_s \right \vert x \in \Gc\cap \Xc_g]\Pr[\Xc_g|\Gc] + \Eb_{x}\left[ \left\lVert \zeta(x, c)- c  \right\rVert_s \right \vert x \in \Gc\cap \Xc_b]\Pr[\Xc_b|\Gc]  &\textrm{(Tower)} \\ 
%     &\leq \epsilon \Pr[\Xc_g |\Gc] + (s+1)^{-1/s} \Pr[\Xc_b|\Gc] \\
%     &\leq \epsilon \frac{\delta'-\delta}{\delta'} + (s+1)^{-1/s} \frac{\delta}{\delta'} & \textrm{(} \epsilon \leq (s+1)^{-1/s} \textrm{)}
%     % &= \left\lVert \Eb_{x}\left[\zeta(x, c)- c  \vert x \in \Xc_g \cap \Gc \right] \frac{\delta'-\delta}{\delta} + \Eb_{x}\left[\zeta(x, c)- c  \vert x \in \Xc_b \cap \Gc \right] \frac{\delta'}{\delta} \right\rVert_s &\textrm{(Tower\ property)} \\ 
%     % &\leq \frac{\delta'-\delta}{\delta} \left\lVert \Eb_{x}\left[\zeta(x, c)- c \vert x \in \Xc_g\cap \Gc \right] \right\rVert_s  + \frac{\delta'}{\delta} \left\lVert\Eb_{x}\left[\zeta(x, c)- c  \vert x \in \Xc_b\cap \Gc \right]  \right\rVert_s &\textrm{(Sub\ additivity)} \\ 
%     % &\leq  \frac{\delta'-\delta}{\delta} 
% \end{align*}
If we don't care about constants too much, we can further simplify above by 
\[ \epsilon \frac{\delta'-\delta}{\delta'} + (p+1)^{-1/p} \frac{\delta}{\delta'} \leq \epsilon + \delta/\delta' \]
% Because $f$ is monotonic, it is invertible. First observe that for any $x \in [0, 1]$ and corresponding $y=f(x)$ we have 
% \[ f(x) - x = y - f^{-1}(y) \]
% Then we have
% \begin{align*}
%     \int (f(x) - x)^p dx &= \int (f^{-1}(y) - y)^p dy 
% \end{align*}
% which concludes the proof
\end{proof}

\subsection{Proofs for Section 6}
\label{sec:appendix_decision_making}

\decision* 

\label{sec:appendix_decision_making_proof}
\begin{proof}[Proof of Theorem~\ref{thm:markov}]
Choose any $x \in \Xc$, $h \in \Hc$ and $r \in (0, 1)$. For some action $a$ assume $l(x, \cdot, a)$ is monotonically non-decreasing. We consider the situation where $y < \fd[x]^{-1}(1-r)$, or equivalently $\fd[x](y) < 1-r$, then % $l(x, y, a) \leq l(x, \fs^{-1}[x](c), a)$, then because $l(x, \cdot, a)$ is we have
\[ l_\fd(x) = \int_{y' \in \Yc} l(x, y', a) d\fd[x](y') \geq \int_{y' \geq y} l(x, y', a) d\fd[x](y') \geq l(x, y, a) \int_{y' \geq y}  d\fd[x](y') \geq r l(x, y, a) \]
because the above is true for any $a \in \Ac$, it must also be true for the action $\phi_\fd(x)$. On the other hand, assume $l(x, \cdot, a)$ is monotonically non-increasing, then by a similar argument we get whenever $\fd[x](y) > r$ we have
\[ l_\fd(x) \geq r l(x, y, a) \]

Consider the set $\Sc_r, \Mc_r, \bar{\Mc}_r \subset \Xc \times \Yc \times \Hc$, defined by 
\begin{align*}
    \Sc_r = \lbrace x, y, h \mid l_h(x) \leq r l(x, y, a)  \rbrace, \quad\Mc_r = \lbrace x, y, h \mid \fd[x](y) \leq r  \rbrace, \quad \bar{\Mc}_r = \lbrace x, y, h \mid \fd[x](y) \geq 1-r  \rbrace \\
\end{align*} 
The above results would imply $\Sc_r \subset \Mc_r \cup \hat{\Mc}_r$. But we know that
\[ \Pr[\xs,\ys,\fs \in \Sc_r] \leq \Pr[\xs,\ys, \fs \in \Mc_r \cup \hat{\Mc}_r] \leq 2r \]
taking $k=1/r$ gives us the desired statement.

If $\fs$ is individually calibrated, then it is also adversarial group calibrated by Theorem~\ref{thm:pai_vs_calibration}. Define a function $\zeta: \Xc \times \Hc \to \lbrace 0, 1 \rbrace$ that represents whether $l$ is monotonically non-decreasing or non-increasing in $\Yc$. Then 
\begin{align*}
&\Pr[\xs,\ys,\fs \in \Sc_r] \\
&\leq \Pr[\xs,\ys, \fs \in \Mc_r \mid \zeta(\xs, \fs) = 0]\Pr[\zeta(\xs, \fs) = 0] + \Pr[\xs,\ys, \fs \in \Mc_r \mid \zeta(\xs, \fs) = 1]\Pr[\zeta(\xs, \fs) = 1] \\
&\leq r
\end{align*}
% Then $(x, y) \not\in \Sc_r$ only if $\fs[x](y) > 1-r$. However, by average calibration we know that
% \[ \Pr[\fs[x](y) > 1-r] \leq r \]
\end{proof}

\subsection{Proofs for Appendix}
\eqece*
\begin{proof}[Proof of Proposition~\ref{prop:eq_ece}]
This proposition depends on the following Lemma. 
\begin{lemma}
\label{lemma:inverse_norm}
Let $\phi: [0, 1] \to [0, 1]$ be a monotonic differentiable function such that $\phi(0)=0$ and $\phi(1)=1$. Let $\psi(x) = x$, then for any $1 \leq s \leq +\infty$ we have
\[ \intc \left\lvert \phi^{-1}(c) - \psi(c)\right\rvert^s dc = \intr\left\lvert \phi(r) - \psi(r)\right\rvert^s dr \]
%where $\lVert \cdot \rVert_s$ 
\end{lemma}

First observe that by the monotonicity of $\fcdfs$ we have
\[ \Pr[\fs[\xs](\ys) \leq c] = \intr \mathbb{I}(\fcdfs \leq c) dr = \fcdfs^{-1}(c) \]
We get
\begin{align*}
    \ece(\fs) &= \int_{c=0}^1 \left\lvert \Pr[\fs[\xs](\ys) \leq c] - c \right\rvert dc = \int_{c=0}^1 \left\lvert \fcdfs^{-1}(c) - c \right\rvert dc \\
    &= \intr  \left\lvert \fcdfs(r) - r \right\rvert dr = d_{W1}(\fcdfs, \ucdf)
\end{align*}

% \begin{align*}
%     \intc \left\lvert \zeta(x, c) - c \right\rvert^p dc = \intc \lvert \fcdfd^{-1}(c) - c \rvert^p dc = \intr \lvert \fcdfd(r) - r \rvert^p dr \leq \epsilon^p 
% \end{align*}

Finally we prove Lemma~\ref{lemma:inverse_norm}.
\begin{proof}[Proof of Lemma~\ref{lemma:inverse_norm}]
%\sj{Will remove the condition monotonic and differentiable in the future, Proof credit to Jiaming.} 
Let $[a, b]$ be an interval where $\phi(x)-x$ does not change sign, and $f(a)=a$, $f(b)=b$. Without loss of generality, assume it is positive. Then 
\begin{align*} 
\int_{x=a}^b \lvert \phi(x) - x \rvert^s dx &- \int_{y=a}^b ( f^{-1}(y) - y )^s dy 
= \int_{x=a}^b ( \phi(x) - x )^s dx  - \int_{x=a}^b ( -x + \phi(x) )^s f'(x) dx \\
&= \int_{x=a}^b (\phi(x) - x )^s (f'(x) - 1) dx = \frac{(\phi(x) - x)^{s+1}}{s+1} \vert_a^b = 0
\end{align*}
Let $0=a_1 < a_2 < \cdots < a_n=1$ be a set of points where $f(a_i)=a_i$, and $f$ does not change sign between $[a_i, a_{i+1}]$. Then we have
\begin{align*}
    \int_{x=0}^1 & \lvert \phi(x) - x \rvert^s dx - \int_{y=0}^1 \lvert f^{-1}(y) - y \rvert^s dy \\ 
    &= \sum_i \left( \int_{x=a_i}^{a_{i+1}} \lvert \phi(x) - x \rvert^s dx - \int_{y=a_i}^{a_{i+1}} \lvert f^{-1}(y) - y \rvert^s dy \right) = 0
\end{align*}
\end{proof}
\end{proof}

\end{document}